\documentclass[twoside]{article}
\DeclareUnicodeCharacter{2113}{\ensuremath{\ell}}

\usepackage[accepted]{aistats2025} %
\usepackage{joeyzhou}
\usepackage{times}
\usepackage[utf8]{inputenc}
\usepackage[T1]{fontenc}
\usepackage{dsfont}
\usepackage[hidelinks]{hyperref}
\usepackage{url}
\usepackage[round]{natbib}

\usepackage{booktabs}
\usepackage{amsfonts}
\usepackage{nicefrac}
\usepackage{microtype}
\usepackage{xcolor}
\usepackage{thm-restate}
\usepackage{graphicx}
\usepackage{tabularx}    %
\usepackage{booktabs}    %
\usepackage[font=small]{caption} %
\usepackage{svg}
\usepackage{subcaption}
\usepackage{tikz}
\usepackage{pgfplots}
\usepgfplotslibrary{fillbetween}
\pgfplotsset{compat=1.17}

\newtheorem{theorem}{Theorem}[section]
\newtheorem{assumption}[theorem]{Assumption}
\newtheorem{definition}[theorem]{Definition}
\newtheorem{lemma}[theorem]{Lemma}
\newtheorem{corollary}[theorem]{Corollary}
\newtheorem{proposition}[theorem]{Proposition}
\newtheorem{example}[theorem]{Example}

\begin{document}

\twocolumn[

\aistatstitle{Near-Polynomially Competitive Active Logistic Regression}

\aistatsauthor{ Yihan Zhou \And Eric Price \And Trung Nguyen }

\aistatsaddress{The University of Texas at Austin \And  The University of Texas at Austin \And The University of Texas at Austin}]

\begin{abstract}%
We address the problem of active logistic regression in the realizable setting. It is well known that active learning can require exponentially fewer label queries compared to passive learning, in some cases using $\log \frac{1}{\eps}$ rather than $\poly(1/\eps)$ labels to get error $\eps$ larger than the optimum.

We present the first algorithm that is polynomially competitive with the optimal algorithm on every input instance, up to  a factor $d^2$ and  factors polylogarithmic in the error and domain size.    In particular, on a simple family of logistic threshold instances with a uniform distribution, our algorithm uses $\polylog(1/\eps)$ labels, whereas every passive learner requires $\Omega(1/\eps^2)$ labels.   Our algorithm is based on  sampling and can be extended to learn more general class of functions. We further support our theoretical results with experiments demonstrating performance gains for logistic regression compared to existing active learning algorithms.

\end{abstract}

\section{INTRODUCTION}
Active learning is a learning paradigm where unlabeled data is abundant and inexpensive, but obtaining labels is costly or time-consuming. The goal is to use as few labeled examples as possible to train an effective model. Unlike passive learning, where the learner receives a fixed set of labeled data, active learning allows the learner to choose which data points to query for labels. It is known that for binary classifiers, active learning algorithms can achieve exponentially better label complexity bounds than passive algorithms in some cases \citep{cohn1994improving}. In this paper, we study active logistic regression, which is a special case of active probabilistic classification. Active probabilistic classification extends binary classification by allowing each hypothesis to provide a labeling probability instead of a deterministic label. Formally, the problem is defined as follows.

\subsection{Problem Definition and Motivation}
\label{subsec:moti}
Let $X\subseteq\R^d$ be a finite dataset, $Y=\curly{0,1}$ be the set of labels and $H$ be a hypothesis class containing hypotheses $h:X\rightarrow [0,1]$. In other words, for each $x\in X$ and $h\in H$, $h(x)$ denote the probability that $x$ has the label $1$. Let $\calD_X$ be a distribution on $X$ and then we use the weighted \(\ell_2\)-distance $\norm{h_1-h_2}_2^{\calD_X} \coloneqq \sqrt{\mathbb{E}_{x \sim \mathcal{D}_X} \left[ \left( h_1(x) - h_2(x) \right)^2 \right]}$ with respect to the distribution $\calD_X$ to measure the distance between two hypotheses \(h_1\) and \(h_2\). Usually, the distribution $\calD_X$ is clear in context and we can simplify the notation by dropping the superscript. There is a ground truth $h^*\in H$ defines the true marginal distribution, i.e., $\Pr\squr{x\text{ has label }1}=h^*(x)$ for every $x\in X$. We define the error of a hypothesis $h\in H$ by $\err\bra{h}\coloneqq\norm{h-h^*}_2$. We define the problem of realizable active probabilistic classification.
\begin{definition}[Realizable Active Probabilistic Classification]
Given a finite dataset $X$ and a marginal distribution $\calD_X$ over $X$. The environment chooses a hidden hypothesis $h^* \in H$ and provides an oracle $\mathcal{O}_{h^*}: X \rightarrow \{0,1\}$ such that for any query $x \in X$, $\mathcal{O}_{h^*}(x)$ returns  an independent label, equal to  $1$ with probability $h^*(x)$ and $0$ otherwise. The player knows $X$ and $\calD_X$ and can query the oracle with any $x \in X$ multiple times. The player's goal is to identify a hypothesis $\hat{h} \in H$ with $\err(\hat{h}) \leq \varepsilon$ with probability at least $1 - \delta$, using as few queries as possible.
\end{definition}
We use a tuple $P\coloneqq\bra{X,\calD_X,H,\varepsilon,\delta}$ to represent a problem instance. We refer to the number of queries made to the oracle as the \emph{label complexity}. We could characterize the hardness of a specific problem instance by the optimal label complexity, defined in the following.
\begin{definition}[Optimal Label Complexity]
    Given a problem instance \( P = \left( X, \mathcal{D}_X, H, \varepsilon, \delta \right) \), we say that an algorithm \( A \) solves \( P \) with \( m_A(P) \) queries if, for every \( h^* \in H \), the algorithm \( A \) uses at most \( m_A \)(P) queries and, with probability at least \( 1 - \delta \), returns a hypothesis \( \hat{h} \) such that \( \err\bra{\hat{h}} < \varepsilon \). The optimal label complexity of \( P \) is the minimal value of \( m_A(P) \) over all algorithms \( A \) that solve \( P \), denoted by $m^*\bra{P}$.
\end{definition}

In this paper, we focus on a special case of this problem--active logistic regression. In logistic regression, each hypothesis $h$ is parameterized by some $\theta\in\R^d$ such that $h(x)=\sigma\bra{\theta^Tx}$, where $\sigma(a)=\frac{1}{1+e^{-a}}$ is the sigmoid function. To make the connection, we use $h_\theta$ to denote the hypothesis $h$ parameterized by $\theta$ and use $\theta_h$ to denote the vector $\theta$ parameterizing $h$. Similarly, we use $H_\Theta$ to denote the hypothesis class parameterized by a class of vectors $\Theta$ and $\Theta_H$ to denote the class of vectors parameterizing the hypothesis class $H$. We call $\Theta_H$ the parameter space and drop the subscript when the context is clear. Throughout this paper, we make the following standard boundedness assumption.
\begin{assumption}[Boundedness Assumption]\label{assumption:boundedness}
In logistic regression, both of the dataset and parameter space is bounded, more precisely,

1. The parameter space   $\Theta\subseteq\R^d$ is compact and convex with nonempty interior, and  $\norm{\theta}_2\le R_1$ for every $\theta\in\Theta$, where $R_1\ge1$. 

2. The dataset $X$ is upper bounded by $R_2\ge1$, i.e., $\norm{x}_2\le R_2$ for every $x\in X$.
\end{assumption}
Similar to the improvement active learners shows on learning binary classifiers, active logistic regression algorithms can significantly outperform their passive counterparts, sometimes by exponential factors, as demonstrated in the following example.
\begin{example}\label{exp:exponential}
 
\begingroup
Let $N\ge4$ be an integer, let $X=\{1,\ldots,N\}$, and let $\mathcal D_X$ be uniform on $X$. We suppress the constant feature coordinate and write $h_{a,t}(i)=\sigma\bigl(a(i-t)\bigr)$, where $a\in[1,2]$ and $t\in[0,N]$ are unknown. Set $\varepsilon=1/(20\sqrt N)$ and $\delta=0.1$.
\endgroup
\end{example}

\begingroup
An active learner can locate $t$ within a constant distance by binary search with repeated queries, then fit a hypothesis using a constant number of additional labels near $t$. This achieves error $\varepsilon$ using $O(\log N\log\log(eN))$ labels. The same bound holds after any fixed constant reduction in the target error or failure probability.

For passive learning, compare $t=s$ and $t=s+1$ with $a=1$ and an integer $1\le s<N$. These hypotheses are more than $2\varepsilon$ apart, so a learner achieving error $\varepsilon$ must distinguish them. Their label probabilities differ substantially only at $O(1)$ points near $s$, with the differences decaying exponentially farther away. Under uniform sampling, collecting enough information to distinguish the two values of $t$ therefore requires $\Omega(N)=\Omega(1/\varepsilon^2)$ labels, giving an exponential separation. A formal proof is given in Appendix~\ref{appendix:exponential}.
\endgroup

\subsection{Our Results}
Our paper has the following contributions:

1. We present the first active logistic regression algorithm with a provable competitive label complexity upper bound, as shown below.
\begin{restatable}{theorem}{main}\label{thm:finaltheorem}
Let   $0<\varepsilon\le1$ and
\[
m = \max\left\{1,m^*\bra{ X, \calD_X, H, \frac{\varepsilon}{2}, 0.01 }\right\}.
\] 
Under Assumption~\ref{assumption:boundedness}, Algorithm~\ref{Alg:RealMain}, with parameter $m$,  returns a hypothesis   $\hat{h}\in H$  such that $\err\bra{ \hat{h} } \leq 17 \varepsilon$ with probability at least $0.7$, using a label complexity of
  \[
O\bra{d^2\,\poly\bra{ m } \, \polylog\bra{ \dfrac{d R_1 R_2 }{ \varepsilon } } }.
\] 
\end{restatable}
This bound implies that Algorithm~\ref{Alg:RealMain} achieves a label complexity that is polynomially competitive with the optimal on any problem instance, up to  a factor $d^2$ and  some polylogarithmic factors in the accuracy and domain size. Furthermore, it demonstrates an exponential improvement over passive algorithms on certain instances, such as Example~\ref{exp:exponential}.

2. Our algorithm is simple. In Section~\ref{sec:exp}, we  implement a practical variant and  conduct experiments demonstrating its performance, showing our algorithm has potential in real-life application.

3. Our algorithm and analysis can be extended to a wider class of probabilistic binary classifiers   with Lipschitz link functions. This extension is discussed in Section~\ref{subsec:extension}.

\section{RELATED WORK}
To the best of our knowledge, our result provides the first known competitive label complexity upper bounds for active logistic regression methods. Our method is adaptive, which means that each query can be chosen based on the outcomes of previous queries. Related works in this area fall into four categories: (i) theory for \emph{passive} logistic regression; (ii) theory for \emph{non-adaptive} active logistic regression; (iii) theory for active learning methods \emph{other than} logistic regression; and (iv) \emph{empirical} results for active logistic regression.

Logistic regression and, more broadly, generalized linear models (GLMs) have been extensively studied in the passive setting. 
Efficient algorithms for learning GLMs with small $\ell_2$ error include ISOTRON~\citep{kalai2009isotron}, GLMtron~\citep{kakade2011efficient}, and Sparsitron~\citep{klivans2017learning}.  One can also estimate the parameters under distributional assumptions~\citep{hsu2024sample}. All the aforementioned algorithms are designed for passive learning; therefore, as shown in Example~\ref{exp:exponential}, they could be exponentially worse than active learning algorithms in some problem instances. 

There is a line of work~\citep{munteanu2018coresets,mai2021coresets,gajjar2023active,gajjar2024agnostic,Chowdhury_Ramuhalli_2024} that employs non-adaptive sampling methods—such as leverage score sampling or Lewis weights—to solve logistic regression, and these techniques can be extended to active learning. However, their results are not directly comparable to ours due to differences in both the setting and the error measures. Specifically, they assume that the dataset $X$ and labels $y$ are fixed in advance, so the labels do not exhibit randomness, and the error bounds they derive are neither simply additive~\citep{gajjar2023active,gajjar2024agnostic,Chowdhury_Ramuhalli_2024} nor measured in the $\ell_2$-distance~\citep{munteanu2018coresets,mai2021coresets}. In addition, while non-adaptive methods offer the advantage of faster implementation, their lack of adaptivity makes them unlikely to achieve the   instance-dependent competitive bounds studied here.

In the active setting, the theory and algorithms for binary classifiers---where each hypothesis maps every \( x \) to a label deterministically---is also well-developed. One class of such algorithms is called disagreement-based active learning, which involves only sampling from the disagreement region, where there exist two hypotheses that have different label predictions \citep{cohn1994improving,balcan2006agnostic,hanneke2007bound,dasgupta2007general}. However, in the setting of probabilistic classification, it is not even clear how to define the notion of a disagreement region, so this class of algorithms does not apply. Another class of active learning algorithms is called splitting-based algorithms, where the algorithm quantifies the informativeness of each point and queries the most informative ones \citep{dasgupta2004analysis,katz2021improved,price2023competitive}. Our algorithm falls into this category and can be seen as an extension from deterministic to probabilistic binary prediction. Our algorithm and analysis draw inspiration from the work of \citet{price2023competitive}, who developed an algorithm with competitive label complexity bounds for active binary classification. In essence, their approach employs the multiplicative weights framework, which assigns a prior over the hypothesis space. At each iteration, the algorithm selects a set of the most informative points—determined by the current prior—and penalizes hypotheses that yield incorrect predictions. Their analysis shows that, after a sufficient number of queries, the posterior distribution concentrates on the ground truth. However, their method is limited to finite binary hypothesis classes and is not directly applicable to the infinite hypothesis spaces encountered in logistic regression, nor does it naturally extend to probabilistic settings. In contrast, our work overcomes these limitations by adapting the approach to active logistic regression and more general function classes. It is also worth noting that active regression is closely related to active probabilistic classification. However, most existing work in this area---for example, \citet{sabato2014active,chen2019active,musco2022active}---has focused on linear regression, and therefore does not extend to logistic regression.

For our setting of active logistic regression, \citet{yang2018benchmark} conducted a comprehensive survey of various active learning algorithms and heuristics, benchmarking their empirical performance. 
However, none of the algorithms considered have label complexity bounds or mathematically rigorous performance guarantees for logistic regression.

\section{ALGORITHM: FIRST ATTEMPT}
We begin to introduce our algorithm in this section. All the omitted proofs in the paper can be found in Supplementary Materials Section~\ref{appendix:analysisproof}. In our initial approach, we try to apply the multiplicative weights framework directly. We start by assigning an initial weight \( w_1(h_\theta) = 1 \) to every hypothesis \( h_\theta \) parameterized by \( \theta \in \Theta \), where   $\Theta\subseteq\R^d$. Normalizing these weights gives a uniform distribution $\lambda_1$ over $\Theta$, which induces a distribution over hypotheses in $H$. At each iteration \( i \), given the current prior \( \lambda_i \), we define the informativeness of each point \( x \in X \) as follows:
\[
r_{\lambda_i}(x) \coloneqq \mathbb{E}_{h \sim \lambda_i} \left[ \mathrm{D}_{\mathrm{KL}}\big( \bar{h}_{\lambda_i}(x) \,\Vert\, h(x) \big) \right],
\]
where \( \bar{h}_{\lambda_i}(x) = \mathbb{E}_{h \sim \lambda_i}[h(x)] \) is the average prediction under \( \lambda_i \), and \( \mathrm{D}_{\mathrm{KL}}(p \Vert q) = p \log \frac{p}{q} + (1 - p) \log \frac{1 - p}{1 - q} \) is the binary Kullback-Leibler divergence (KL divergence). This function \( r_{\lambda_i}(x) \) measures the expected KL divergence between the average prediction \( \bar{h}_{\lambda_i}(x) \) and individual predictions \( h(x) \), indicating the level of disagreement among hypotheses at point \( x \). If the prior $\lambda$ is not overly concentrated, we can relate the information function \( r \) of the most informative point to the optimal label complexity \( m^*(X, \mathcal{D}_X, H, \varepsilon, 0.01) \). Let $B_h(\varepsilon)$ be the ball centered at $h$ with radius $\varepsilon$ in the hypothesis space. This relationship is formalized in one of our core lemmas below.
\begin{lemma}[Lower Bound for Non-concentrated Distribution]\label{lemma:mstarrelationstrong}
If   $\lambda$  is a distribution over   $H$  such that no   $h\in H$ satisfies $\lambda(B_h(2\varepsilon))>0.8$, then, for any $0<\delta\le0.1$,
\[
\max_{x \in X} r_\lambda(x) \gtrsim \frac{1}{m^*(X, \mathcal{D}_X, H, \varepsilon, \delta)}.
\] 
\end{lemma}
Note that computing \( r_{\lambda_i}(x) \) exactly is computationally intensive due to the expectation over \( \lambda_i \). Instead, we approximate it by sampling hypotheses \( h \) from \( \lambda_i \) and estimating \( r_{\lambda_i}(x) \) using these samples. After the estimation, we query the most informative point \( x_i \), i.e., the one with the highest estimated \( r_{\lambda_i}(x) \), and obtain the corresponding label \( y_i \). For each query-label pair \( (x_i, y_i) \) and hypothesis \( h \), we use the cross-entropy loss as the penalty:
\[
\ell_{h}(x_i, y_i) = y_i \log \frac{1}{h(x_i)} + (1 - y_i) \log \frac{1}{1 - h(x_i)}.
\]
We update the weight of each hypothesis \( h_\theta \) as:
\[
w_{i+1}(h_\theta) \coloneqq w_i(h_\theta) \cdot \exp\big( -\ell_{h_\theta}(x_i, y_i) \big).
\]
Normalizing the weights gives an updated probability density function (PDF) of the distribution $\lambda_{i+1}$, but we don't have this normalization step in our algorithm because it is costly and unnecessary. As more queries are made, we expect the distribution \( \lambda_i \) to concentrate around the true hypothesis \( h^* \). Therefore, at the end, we simply sample \( \hat{h} \) from the final distribution   \( \lambda_{K+1} \). This algorithm is summarized in Algorithm~\ref{Alg:ActiveSimple}.

\begin{algorithm2e}[htbp]
\SetAlgoLined
\DontPrintSemicolon
\SetKwProg{Proc}{Algorithm}{}{}
\caption{Active Logistic Regression: First Attempt}\label{Alg:ActiveSimple}
\Proc{\textsc{ActiveSimple}$(P,K)$}{
Initialize \( w_1(h_\theta) = 1 \) for every \( \theta \in \Theta \)\;
\For{\( i = 1 \) \KwTo \( K \)}{
    Estimate \( r_{\lambda_i}(x) = \mathbb{E}_{h \sim \lambda_i} \left[ \mathrm{D}_{\mathrm{KL}}\big( \bar{h}_{\lambda_i}(x) \,\Vert\, h(x) \big) \right] \) for all \( x \in X \) using samples \( h \sim \lambda_i \), obtaining estimates \( \hat{r}_{\lambda_i}(x) \)\;
    Select \( x_i = \arg\max_{x \in X} \hat{r}_{\lambda_i}(x) \)\;
    Query \( x_i \) and receive label \( y_i \)\;
    Update weights:\;
    \quad \( w_{i+1}(h_\theta) = w_i(h_\theta) \cdot \exp\big( -\ell_{h_\theta}(x_i, y_i) \big) \) for all \( h_\theta \in H \)\;
}
  \Return{\( \hat{h} \sim \lambda_{K+1} \)}
 }
\end{algorithm2e}

\subsection{Analysis Attempt}\label{subsec:firstattemp}
As mentioned earlier, we expect the distribution \( \lambda \) to concentrate around \( h^* \) as more queries are made. To formalize this, we define the potential \( \psi_i(h^*) \coloneqq \log \lambda_i(h^*) \). Ideally, we aim to relate the growth in potential to the information function \( r \), and show that \( r \) is lower bounded, ensuring progress at each iteration. We begin by calculating the expected potential growth, as outlined in the following lemma.
\begin{lemma}\label{lemma:singlequerypotentialgain}
Let $\lambda_i$ be the prior distribution at iteration $i$, $x_i$ be the queried point, and $y_i$ be its label. Then the expected potential gain is
\[
\mathbb{E}_{y_i}\left[ \psi_{i+1}(h^*) - \psi_i(h^*) \,\middle|\, x_i \right] = D_{\mathrm{KL}}\left( h^*(x_i) \,\big\Vert\, \bar{h}_{\lambda_i}(x_i) \right).
\]
\end{lemma}

Since the KL divergence is non-negative, we have the nice property that no matter which point we query, the expected potential growth is always non-negative. This property is useful for our analysis. However, several challenges prevent us from proving that Algorithm~\ref{Alg:ActiveSimple} converges to $h^*$ quickly enough:

1.\phantomsection\label{problem1} The KL divergence term depends only on the mean of the distribution. This brings a problem that even if the algorithm queries an informative point, the expected potential gain could still be low, even zero. Consider the following example. Suppose the prior $\lambda$ assigns a probability of $0.01$ to $h^*$ and distributes the remaining probability evenly between hypotheses $h_1$ and $h_2$, where $h^*(x)=\frac{1}{2}$, $h_1(x) = 1$ and $h_2(x) = 0$. Here $x$ is an informative point because most hypotheses ($h_1$ and $h_2$) disagree with $h^*$ on $x$, suggesting that querying $x$ should significantly increase $\lambda(h^*)$. However, the mean prediction under $\lambda$ is $\frac{1}{2}$, matching $h^*(x)$, which results in zero expected potential gain.
Thus on this example, Algorithm~\ref{Alg:ActiveSimple} chooses a good point to query but the potential does not grow.

2.\phantomsection\label{problem2} Conversely, Algorithm~\ref{Alg:ActiveSimple} may also select an \emph{uninformative} point to query, if the distribution is overconcentrated. The algorithm chooses the query point to maximize $r(x)$, which
Lemma~\ref{lemma:mstarrelationstrong} shows how to relate to   the optimal label complexity.  But the lemma requires that $\lambda$ not be too concentrated in a small region; if our prior $\lambda$ on hypotheses is highly concentrated, the lemma does not apply and we cannot show that the query point is informative because it is not true.

3.\phantomsection\label{problem3} The KL divergence is unbounded, introducing complications when attempting to establish concentration inequalities and a high probability bound.

\section{ALGORITHM: REFINEMENT}\label{sec:MainAlgo}
To address the issues described in Section~\ref{subsec:firstattemp}, we refine our algorithm as follows.

\subsection{Double Query}

To address challenge~\hyperref[problem1]{1}, we modify our algorithm so that in each iteration, instead of querying \( x_i \) once, we query it twice and obtain labels $y_i^1$ and $y_i^2$. This adjustment allows us to relate the expected potential growth to the mean and the variance.

\begin{lemma}\label{lemma:mainpotentialgrowthlb}
The expected potential growth in iteration \( i \), conditioned on \( x_i \) being queried twice, is bounded below by
\begin{align*}
&\E_{y_i^1, y_i^2}\left[ \psi_{i+1}(h^*) - \psi_i(h^*) \,\big|\, x_i \right]\\
\gtrsim &\left( h^*(x_i) - \bar{h}_{\lambda_i}(x_i) \right)^2 + \left( \var_{h \sim \lambda_i} [ h(x_i) ] \right)^2.
\end{align*}
\end{lemma}
In the example given in challenge~\hyperref[problem1]{1}, the variance is large, so by performing the double query, the expected potential gain is substantial, as desired. In fact, in Lemma~\ref{lemma:Relationfmstar} we can lower bound this variance term by a polynomial in the information function \( r_{\lambda_i} \). When $\lambda$ is not overconcentrated, we can then apply Lemma~\ref{lemma:mstarrelationstrong} to show that the best query $x$ has expected potential growth   bounded below by an inverse polynomial in the optimal label complexity.

\subsection{Sampling Procedure}

To address Challenge~\hyperref[problem2]{2} and prevent the algorithm from stalling due to overconcentration, we use the sampling procedure given in Algorithm~\ref{Alg:Sampling}. This procedure samples \( h_1 \) from one distribution \( p \), then rejection samples \( h_2 \) from another distribution \( q \) such that   \( \norm{h_1 - h_2} >2\varepsilon \). It then uniformly randomly outputs \( h_1 \) or \( h_2 \). The resulting distribution is:
  \[
\hat{\lambda}(h) \coloneqq \frac{1}{2} p(h) + \frac{1}{2} \mathbb{E}_{h' \sim p} \left[ q^{H \setminus B_{h'}(2\varepsilon)}(h) \right],
\] 
where   \( q^{H \setminus B_{h'}(2\varepsilon)} \)  denotes the conditional distribution of \( q \) outside the ball   \( B_{h'}(2\varepsilon) \). It can be shown that under the distribution \( \hat{\lambda} \), no ball of radius \( \varepsilon \) in the hypothesis space has probability mass more than \( 0.8 \), provided we set the parameters of the sampling procedure properly.

Sampling from Algorithm~\ref{Alg:Sampling}, we could relate the information function $r$ to   the optimal label complexity as desired by Lemma~\ref{lemma:mstarrelationstrong}. However, at the same time, this sampling procedure introduces new questions that need to be answered. How do we choose the distributions \( p \) and \( q \)? How do we relate the information function \( r \) with respect to \( \hat{\lambda} \) to the potential change? These questions are indeed tricky. To answer these questions and facilitate our analysis, we use a nested loop structure in our refined algorithm. We refer to the outer loop iterations as ``phases'' and the inner loop iterations as ``iterations''. We denote the \( j \)-th iteration in phase \( i \) by \( (i, j) \).

In the refined algorithm, at the beginning of each phase \( i \), we fix a distribution \( \lambda^0 = \lambda_i \), which remains unchanged during the phase. We also initialize a distribution \( p_{(i, 1)} \) to be uniform over \( \Theta \) at the start of the phase. Then, in iteration \( j \), we use the sampling procedure with \( p = \lambda^0 \) and \( q = p_{(i, j)} \) and query the most informative point \( x_{(i, j)} \) with respect to the distribution \( \hat{\lambda}_{(i, j)} \coloneqq \frac{1}{2} \lambda^0 + \frac{1}{2} \lambda^1_{(i, j)} \), where   \( \lambda^1_{(i, j)} \coloneqq \mathbb{E}_{h' \sim \lambda^0} \left[ q^{H \setminus B_{h'}(2\varepsilon)}(h) \right] \). After getting two labels, we update $p_{(i,j+1)}$ at the end of the iteration. At the end of the phase, we query all of the points  in the multiset  \( \{ x_{(i, 1)}, \ldots, x_{(i, M)} \} \) twice again and use the fresh labels to update \( \lambda_{i+1} \). As shown later in Section~\ref{sec:Analysis}, we expect that in most of the phases, the last iteration distribution \( p_{(i, M+1)} \) concentrates around \( h^* \), so we return \( \hat{h} \) by sampling from the average of the last iteration of each phase.

\subsection{Clipping}\label{subsec:clipping}

To address the unboundedness issue in Challenge~\hyperref[problem3]{3}, we clip the hypothesis class so that each hypothesis \( h \in H_\gamma \) satisfies \( h(x) \in [\gamma, 1 - \gamma] \) for all \( x \in X \) and \( \gamma \in (0, \frac{1}{2}) \). Specifically, we define
\[
H_\gamma = \left\{ h': h'(x) = \mathrm{clip}\big( h(x), \gamma \big), \forall h \in H \right\},
\]
where \( \mathrm{clip}(z, \gamma) = \min\left\{ \max\left\{ z, \gamma \right\}, 1 - \gamma \right\} \). The refined algorithm then operates on the clipped hypothesis class \( H_\gamma \).

Clipping may seem problematic, particularly since \( h^* \) may not lie in \( H_\gamma \). In Section~\ref{subsec:removeclipping}, we address this issue by providing a black-box reduction from unclipped to clipped instance. Additionally, clipping does not change the parameter space, and can be applied directly to \( h_\theta \). The refined algorithm is shown in Algorithm~\ref{Alg:Main}.

\begin{algorithm2e}[htbp]
\SetAlgoLined
\DontPrintSemicolon
\SetKwProg{Proc}{Procedure}{}{}
\Proc{\textsc{SamplingProc}$(p, q, \varepsilon)$}{
    Sample \( h_1 \) according to the fixed distribution \( p \)\;
    Rejection sample \( h_2 \) according to \( q \) until   \( \norm{h_1 - h_2}>\varepsilon \) \;
    \Return{\( h_1 \) with probability \( 0.5 \) and \( h_2 \) with probability \( 0.5 \)}\;
}
\caption{Sampling Procedure}\label{Alg:Sampling}
\end{algorithm2e}

\section{ANALYSIS OF ALGORITHM~\ref{Alg:Main}}\label{sec:Analysis}
In this section, we provide an overview of the proof for the label complexity of Algorithm~\ref{Alg:Main}. In the implementation of Algorithm~\ref{Alg:Main}, we sample hypotheses $h$ to estimate the informativeness function $r$. For the purpose of analysis, we simplify by assuming that we can compute $r$ exactly, thereby neglecting the estimation error.   Since sampling hypotheses requires no label queries, we can estimate $r$ accurately enough, with high probability, to preserve the label-complexity bound up to constant factors.  As mentioned in Section~\ref{subsec:clipping}, the hypothesis class is clipped, and we keep the realizable assumption by letting the true hypothesis $h^*\in H_\gamma$. Formally, we make the following clipping assumption throughout this section.
\begin{assumption}[Clipping Assumption]\label{assumption:clipping}
For every   $x\in X$ and $h\in H_\gamma$, $h(x)\in[\gamma,1-\gamma]$, where $0<\gamma<1/4$.
\end{assumption}

Let's also define some notations here. Let $\curly{\calF_{(i,j)}}_{i\in[K],j\in[M]}$ be a filtration and $\calF_{(i,j)}$ be the $\sigma$-algebra of all the queried points up to $(i,j)$, all the labels for $p$   before iteration  $(i,j)$ and all the labels for $\lambda$ up to the previous phase $i-1$. Also recall that we define the potential $\psi_i(h^*)=\log\lambda_i(h^*)$.  For the clipped analysis, let $m^*=\max\{1,m^*(X,\calD_X,H_\gamma,\varepsilon/2,0.02)\}$. Throughout the analysis, $C>0$ denotes a sufficiently large universal constant whose value may change from one occurrence to the next. Now we analyze the potential change in a more fine-grained fashion. If in the current phase $(i,j)$, there exists some queried point $x_{(i,j)}$ satisfies the property that sampling from   $\lambda^0=\lambda_i$, with high probability, $\DKL\bra{h^*\bra{x_{(i,j)}}\Vert h\bra{x_{(i,j)}}}$ is not small, i.e., a non-trivial proportion of hypotheses is not too close to $h^*$ on the queried point $x_{(i,j)}$, then we could expect   $\lambda_i(h^*)$  grows by a non-trivial amount on this iteration. This observation is formal characterized by the following lemma.
\begin{lemma}\label{lemma:goodcaseexpectation}   For sufficiently large $C$, let $\zeta=\frac{1}{C(m^*)^4\log^5(1/\gamma)}$  and let $A_{(i,j)}$ be the event that
  \begin{align*}
&\Pr_{h \sim \lambda^0}\left[ \DKL\left( h^*\bra{x_{(i,j)}}\Big\Vert h\bra{x_{(i,j)}} \right) \ge \zeta \middle|\calF_{(i,j)} \right]\\
&\ge\frac{1}{C(m^*)^4 \log^4 \frac{1}{\gamma}},
\end{align*} 
then
  \[
\E\squr{\psi_{i+1}(h^*)-\psi_i(h^*)|\calF_{(i,j)},A_{(i,j)}}\gtrsim\frac{1}{\bra{m^*}^{24}\log^{32}\frac{1}{\gamma}}.
\] 
\end{lemma}
Otherwise, on every queried point $x_{(i,j)}$ in this phase, the vast majority of $h$ could be quite close to $h^*$, so the potential growth could be slow for the entire phase. Fortunately, in this case, we could show that $p_{(i,j)}(h^*)$ grows fast relative to the hypotheses that are some distance away from $h^*$. We introduce a new alternative potential $\tilde{\psi}$ to facilitate such intuition. Let $\tilde{p}_{(i,j)}^{H\setminus B_{h'}(2\varepsilon)}(h)=\dfrac{p_{(i,j)}(h)}{p_{(i,j)}\bra{H\setminus B_{h'}(2\varepsilon)}}$. Note that $\tilde{p}_{(i,j)}^{H\setminus B_{h'}(2\varepsilon)}$ is not a proper PDF. We then define the alternative potential $\tilde{\psi}_{(i,j)}(h^*)\coloneqq\E_{h'\sim  \lambda^0}\squr{\log\tilde{p}_{(i,j)}^{H\setminus B_{h'}(2\varepsilon)}(h^*)}$. Similarly as Lemma~\ref{lemma:mainpotentialgrowthlb}, we could lower bound the expected potential growth of the alternative potential in the following lemma.
\begin{lemma}\label{lemma:altpotentialgrowthlb}
Let 
\begin{align*}
\eta_{(i,j)}&=\bra{h^*\bra{x_{(i,j)}}-\bar{h}_{p_{(i,j)}^{H\setminus B_{h'}(2\epsilon)}}\bra{x_{(i,j)}}}^2\\
&+\bra{\var_{h\sim p_{(i,j)}^{H\setminus B_{h'}(2\epsilon)}}\squr{h\bra{x_{(i,j)}}}}^2.
\end{align*}
Then the conditional expected potential growth in iteration $(i,j)$ is bounded below by
\begin{align*}
&\E\squr{\tilde{\psi}_{(i,j+1)}(h^*) - \tilde{\psi}_{(i,j)}(h^*)\Big|\calF_{(i,j)}}\\
&\qquad\gtrsim\E_{h'\sim  \lambda^0}\squr{\eta_{(i,j)}\big|\calF_{(i,j)}}.
\end{align*}
\end{lemma}
Then in this case, we could lower bound the expected growth of the alternative potential.
\begin{lemma}\label{lemma:badcaseexpectation}
Let $A_{(i,j)}$ be the event defined in Lemma~\ref{lemma:goodcaseexpectation}, then
\begin{align*}
&\E\squr{\tilde{\psi}_{(i,j+1)}(h^*) - \tilde{\psi}_{(i,j)}(h^*) \Big|\calF_{(i,j)},\widebar{A}_{(i,j)}}\\
&\qquad\gtrsim\frac{1}{\bra{m^*}^4 \log^4 \frac{1}{\gamma}}.
\end{align*}
\end{lemma}
Combining Lemma~\ref{lemma:goodcaseexpectation} and Lemma~\ref{lemma:badcaseexpectation} and set the number of iterations $M$ properly, we get the following guarantee for each phase.
\begin{lemma}[Phase Potential Guarantee]\label{lemma:phaseguaranteegoodcase}
At phase $i$,   let $\beta\ge1$ and set
\[
M=\left\lceil C(\beta+1)(m^*)^8\log^{10}\frac{1}{\gamma}\right\rceil
\] 
  for a sufficiently large universal constant $C$. With uniform initialization,  one of the following   holds, conditional on the history before the phase:
\begin{enumerate}
\item   $\Pr[\tilde\psi_{(i,M+1)}(h^*)\ge\beta]\ge0.99$;
 \item   $\E[\psi_{i+1}(h^*)-\psi_i(h^*)]\gtrsim (m^*)^{-24}\log^{-32}(1/\gamma)$.
\end{enumerate}
\end{lemma}

To summarize, we've established that at each phase, either the growth of the potential is lower bounded, or in the ending iteration of the phase, the alternative potential has a high value.

We are close to completing the proof, but one issue remains: the potential is related only to the PDF of $h^*$, while we need to show that a small neighborhood around $h^*$ has high probability. To address this, we take a small-radius ball $B$ around $h^*$ in the parameter space.   For the likelihood-weighted distributions considered here, Lemma~\ref{lemma:ballprobability} relates the  PDF at $h^*$   to the probability of  $B$. In Lemma~\ref{lemma:phaseguaranteegoodcase}, one possible scenario is that the alternative potential of $h^*$ is high, which implies that a small-radius ball around $\theta_{h^*}$ in the parameter space has high probability. This ensures that $h^*$ is contained within a heavy ball, as shown in the following lemma.

\begin{lemma}\label{lemma:tildepsi}
Let \(B\subseteq H\) be inside a ball centered at \(h^*\) with radius less than \(\varepsilon\) in the hypothesis space. If   \(\tilde{\psi}_{(i,j)}(B) \ge 10\)  with respect to any choice of $\lambda^0$, then there exists a ball   $B'$ with radius $4\varepsilon$ such that   $p_{(i,j)}(B')\ge 0.9$  and   $h^*\in B'$.
\end{lemma}
Therefore, we've shown that in each phase, either the potential grows by a decent amount, or the alternative potential is high in the last iteration, which implies $h^*$ is contained in a heavy ball with radius $4\varepsilon$ with high probability. Then by setting the total number of phases $K$ and total number of iterations $M$ properly, we can show that in almost all the phases, the latter happens. As a result, sampling from the averaged distribution of the last iterations will give a hypothesis with high accuracy with high probability. The label complexity of Algorithm~\ref{Alg:Main} is given as follows.
\begin{lemma}\label{lemma:clippedfinal}
  Under Assumptions ~\ref{assumption:boundedness} and ~\ref{assumption:clipping},  suitable choices of $K,M$ in  Algorithm~\ref{Alg:Main}   give $\err(\hat h)\le8\varepsilon$  with probability greater than $0.8$, using
\[
O\left(d^2(m^*)^{32}\log^{42}\frac{1}{\gamma}\,
\log^2\frac{e dR_1R_2m^*}{\varepsilon\gamma}\right)
\]
queries. The same guarantee holds with any supplied upper bound on $m^*$ in the parameter settings.
 \end{lemma}

\begingroup
\begin{algorithm2e}[tb]

\SetAlgoLined
\DontPrintSemicolon
\SetKwProg{Proc}{Algorithm}{}{}
\Proc{\textsc{ClippedActive}$(X, \mathcal{D}_X, H_\gamma, \varepsilon, K, M)$}{
    \If{$\operatorname{diam}(H_\gamma)\le4\varepsilon$}{\Return{any $h\in H_\gamma$}\;}
    Set $\lambda_1$ to be uniform over $\Theta$\;
    \For{phases \( i = 1 \) \KwTo \( K \)}{
        Set $\lambda^0=\lambda_i$\;
        Initialize \( p_{(i, 1)} \) to be uniform over \( \Theta \)\;
        \For{iterations \( j = 1 \) \KwTo \( M \)}{
            Let $\hat\lambda_{(i,j)}$ be the distribution produced by \textsc{SamplingProc}$(\lambda^0,p_{(i,j)},2\varepsilon)$\;
            Query \( x_{(i, j)} \coloneqq \arg\max_{x \in X}r_{\hat{\lambda}_{(i, j)}}(x) \) twice and obtain labels\;
            Update \( p_{(i, j + 1)} \) using the query $x_{(i,j)}$ and the labels\;
        }
        Query every occurrence in the multiset \( X_i \coloneqq \{x_{(i, 1)}, \ldots, x_{(i, M)}\} \) twice again using fresh labels and obtain labels
        $Y_i$\;
        Update \( \lambda_i \) to $\lambda_{i+1}$ using the query multiset \( X_i \) and the corresponding labels \( Y_i \)\;
    }
    \Return{ \( \hat{h} \) by sampling from \( \bar{\lambda} \coloneqq \frac{1}{K} \sum_{i = 1}^K p_{(i, M + 1)} \) }
}
\caption{Active Learning for Logistic Regression on Clipped Instance}\label{Alg:Main}
\end{algorithm2e}
\endgroup

\section{FINAL LABEL COMPLEXITY BOUND}

 We apply the clipped algorithm directly in the original parameter space. The resulting bound retains its explicit dependence on $d$ and compares with the optimum at a constant-factor accuracy. Algorithm~\ref{Alg:RealMain} gives the parameter settings.

\begingroup
\begin{algorithm2e}[tb]

\SetAlgoLined
\DontPrintSemicolon
\SetKwProg{Proc}{Algorithm}{}{}
\caption{Active Learning for Logistic Regression}\label{Alg:RealMain}
\Proc{\textsc{ActiveLogisticRegression}$(P,m)$}{
Set $L=\log(e dR_1R_2m/\varepsilon)$\;
Set $\gamma=\varepsilon/(C d^2m^{32}L^{44})$\;
Construct $H_\gamma=\{\mathrm{clip}(h,\gamma):h\in H\}$\;
Set $K=\lceil C d m^{24}L^{33}\rceil$ and $M=\lceil C d m^8L^{11}\rceil$\;
Let $\hat h\leftarrow$ \textsc{ClippedActive}$(X,\mathcal D_X,H_\gamma,\varepsilon,K,M)$\;
Let $\hat\theta$ be the stored parameter of $\hat h$\;
\Return{$h_{\hat\theta}$}\;
}
\end{algorithm2e}
\endgroup

\subsection{Removing the Clipping Assumption}\label{subsec:removeclipping}

The final step is to eliminate the clipping assumption. By setting $\gamma$ small enough, the discrepancy between clipped and unclipped hypotheses becomes negligible, as the algorithm does not sample enough points for it to matter. The following lemma provides a black-box reduction that relates the sample complexities of the clipped and unclipped instances.

\begin{lemma}[Reduction from Clipped to Unclipped Instances]\label{lemma:blackboxred}
  Suppose  $A$   uses at most $T$ queries and returns a clipped hypothesis with error at most $\eta$ with probability at least  $0.8$   on every realizable instance over $H_\gamma$. If $\gamma T\le0.1$, running  $A$   with the original label oracle and returning the corresponding unclipped hypothesis gives error at most $\eta+2\gamma$ with probability at least  $0.7$.
\end{lemma}

  With the settings in Algorithm~\ref{Alg:RealMain}, the query count is $O(d^2m^{32}L^{44})$ and $\gamma T\le0.1$. A further coupling compares the clipped benchmark with the original benchmark; the details are given in the proof of Theorem~\ref{thm:finaltheorem}.

\main*

\subsection{Discussion}

We address some important considerations related to Theorem~\ref{thm:finaltheorem}.

\paragraph{Extra Polylogarithmic Factors:} Let $G \subseteq H$ be any $2\varepsilon$-packing of the hypothesis class $H$. Since each query can provide at most one bit of information, we require at least $\Omega\left( \log |G| \right)$ queries to solve the problem. While it is not universally true, in many practical and sufficiently complex cases, the size of the largest $2\varepsilon$-packing   grows at least polynomially with $\frac{1}{\varepsilon}$ and $d$. Additionally, the boundedness parameters $R_1$ and $R_2$ often scale linearly with $d$ in these contexts. Therefore, in such natural and complex scenarios, $\log\left( \frac{ d R_1 R_2}{\varepsilon} \right)$ serves as a meaningful lower bound on   $m$. Consequently, the extra polylogarithmic factors in our label complexity   can be absorbed into the polynomial dependence on $m$.

\paragraph{Polynomial Dependence on   $m$:} We believe that our current analysis may not be tight. We conjecture that the actual label complexity of our algorithm exhibits a quadratic or even linear dependence on   $m$, analogous to the results in \citet{price2023competitive} for the deterministic binary classifiers.

\begingroup
\paragraph{Quadratic Dependence on $d$:} We believe that the quadratic dependence on $d$ can be improved. In our analysis, both the number of phases and the number of iterations per phase scale linearly with $d$. One possible direction is to remove the double-loop structure by constructing the sampling distribution more carefully, potentially reducing the dependence from $d^2$ to $d$. Removing the explicit dependence on $d$ entirely seems harder and may require a different algorithm, since our analysis relies on a dimension-dependent bound on the prior mass near the true parameter.
\endgroup

\paragraph{Parameter Setting of Algorithm~\ref{Alg:RealMain}:} Properly setting the parameters $\gamma$, $K$, and $M$ in Algorithm~\ref{Alg:RealMain} appears to require knowledge of   $m$. To determine $\gamma$, we can utilize an upper bound of   $m$ from passive learning algorithms; for instance, the $\tilde{O}\left( \frac{R_1^2  R_2^2 d}{\varepsilon^4} \right)$ upper bound   obtained by rescaling the result of \citet{klivans2017learning}   affects the label complexity only through logarithmic factors. However, setting $K$ and $M$ indeed necessitates knowledge of   $m$. One potential solution is to directly use the computable information function $r$ as a measure of the expected potential growth   to choose $K$ and $M$ adaptively.  Establishing a guarantee for such an adaptive choice remains open.

\paragraph{Boosting the Success Probability:} Algorithm~\ref{Alg:RealMain} succeeds with probability $0.7$. We can amplify its success probability to $1 - \delta$ by running $O\left( \log \frac{1}{\delta} \right)$ independent copies of the algorithm and returning the center of the heaviest $34\varepsilon$ ball among these hypotheses, as established in the following corollary.

\begin{corollary}\label{corollary:highprobability}
    Let   $m$ be as in Theorem~\ref{thm:finaltheorem}.
     Under Assumption~\ref{assumption:boundedness}, there exists an algorithm that returns a hypothesis   $\hat{h}\in H$  such that $\mathrm{err}\left( \hat{h} \right) \leq 68 \varepsilon$ with probability at least $1 - \delta$, using a label complexity of
      \[
    O\left(d^2\,\mathrm{poly}\left( m \right) \mathrm{polylog}\left( \dfrac{dR_1 R_2}{\varepsilon} \right) \log \dfrac{1}{\delta} \right).
    \] 
    \end{corollary}

\section{RUNNING TIME}
For each parameter $\theta_h \in \Theta_H$, the  unclipped  penalty function is
\[
\ell_{\theta_h}(x, y) = y \log \frac{1}{\sigma(\theta_h^\top x)} + (1 - y) \log \frac{1}{1 - \sigma(\theta_h^\top x)},
\]
where $\sigma(\cdot)$ is the sigmoid function. Since this   penalty is convex, the distribution over $\Theta_H$ is log-concave, and prior work shows that sampling from log-concave distributions can be done in polynomial time~\citep{lovasz1993random,kannan1997random,lovasz2006simulated,vempala2019rapid}. Thus,  without clipping,  sampling from $\lambda$ or $p$ takes polynomial time.

 One subtle point is that clipping can make the penalty function non-convex, so the above sampling guarantees do not apply directly. Clipping is nevertheless simple to implement, and efficient sampling may still be possible. One possible way to avoid the loss of log-concavity is to sample using the original logistic weights and clip only the resulting predictions, although this would require modifying the label-complexity analysis.

  A more serious issue with our algorithm's running time is that  it involves polynomially many calls to Algorithm~\ref{Alg:Sampling}, which does rejection sampling to find two hypotheses that are $\eps$ far from each other.  We do not know how to show that this takes polynomial time in general. Although it can be inefficient if $\lambda$ and $p$ concentrate on the same region, we conjecture the double concentration is around the true hypothesis $h^*$ with high probability. So if the rejection step takes too long, the current distribution is likely already concentrated around $h^*$, allowing us to sample directly.

\section{EXTENSIONS}\label{subsec:extension}
  Our current  analysis does not rely on any properties specific to the sigmoid function beyond its Lipschitz continuity.  Moreover, when the corresponding penalty function is convex, we believe that efficient sampling is possible in practice.  We recall the following definition:
\begin{definition}[Lipschitz Continuity]
A function $f:\mathbb{R}\to\mathbb{R}$ is said to be \emph{$L$-Lipschitz continuous} if there exists a constant $L \ge 0$ such that for all $x,y\in\mathbb{R}$,
\[
\lvert f(x)-f(y)\rvert\le L\lvert x-y\rvert.
\]
\end{definition}
Thus, our algorithm and analysis extend naturally to other probabilistic binary functions by replacing the sigmoid with any function   $f:\mathbb{R}\to[0,1]$  that satisfies two conditions: (i) the corresponding penalty function
\[
\ell_\theta(x,y)=y\log\frac{1}{f(\theta^\top x)}+(1-y)\log\frac{1}{1-f(\theta^\top x)}
\]
must be convex, and (ii) the function $f$ must be $L$-Lipschitz continuous. This   includes Bernoulli  generalized linear models   whose link functions satisfy these conditions. In our proofs, the only modification necessary is in Lemma~\ref{lemma:ballprobability}, where the Lipschitz property is used; the corresponding Lipschitz constant will then appear in the final label complexity bound.

\section{EXPERIMENTS}\label{sec:exp}
\setlength{\textfloatsep}{10pt plus 2pt minus 2pt}
\setlength{\dbltextfloatsep}{10pt plus 2pt minus 2pt}

\begin{figure*}[t]
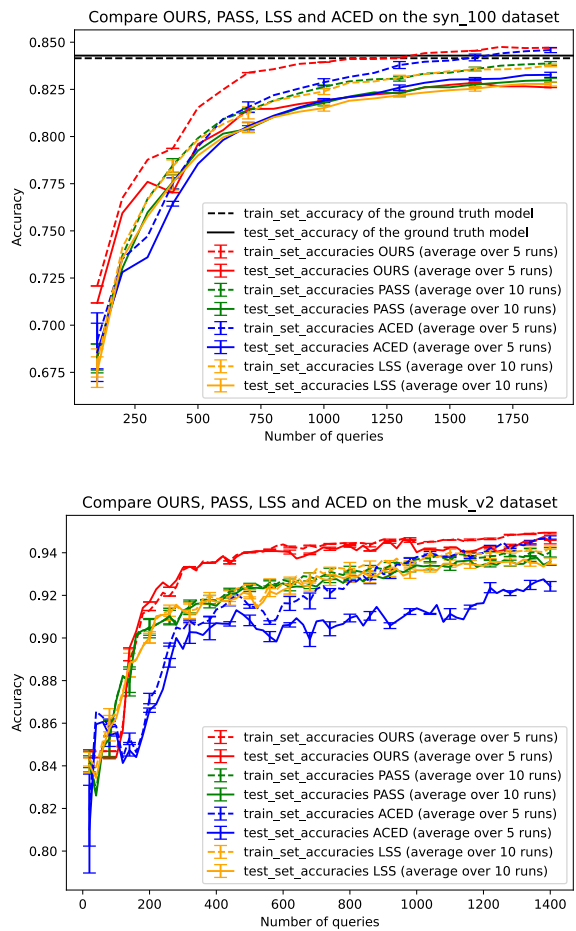

    \centering
    \begin{subfigure}[b]{0.48\textwidth}
        \includesvg[width=0.94\linewidth]{pic/compare_ours_passive_lss_aced_on_syn_100}
        \caption{}
        \label{fig:syn_main_text}
    \end{subfigure}
    \hfill
    \begin{subfigure}[b]{0.48\textwidth}
        \includesvg[width=0.94\linewidth]{pic/compare_ours_passive_lss_aced_on_musk_v2}
        \caption{}
        \label{fig:musk_v2_main_text}
    \end{subfigure}
    \caption{Comparison of OURS with PASS, LSS and ACED on (a) a 100-dimension synthetic dataset and (b) the Musk dataset}
    \label{fig:main}
\end{figure*}

\begin{figure}[t]
    \centering
    \includesvg[width=0.94\linewidth]{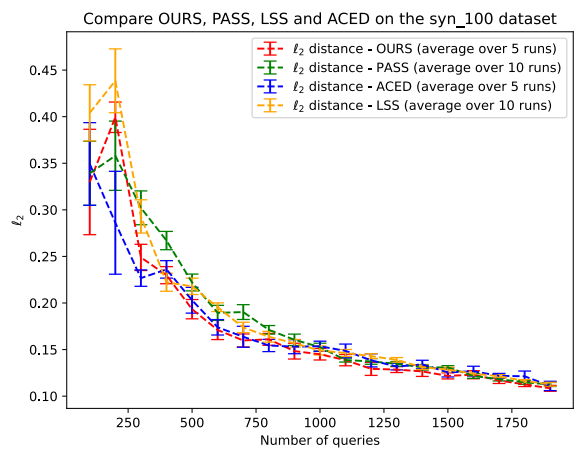}
    \caption{Comparison of OURS with PASS, LSS and ACED in terms of the weighted \(\ell_2\)-distances between estimated hypotheses and the ground truth hypothesis on the synthetic dataset}
    \label{fig:l2}
\end{figure}

We implemented Algorithm~\ref{Alg:ActiveSimple}, clipping the logarithmic ratio at 100 when computing the KL divergence, and refer to this variant as OURS\footnote[2]{The code is available at the following GitHub repository: \url{https://github.com/trung6/ActLogReg}.}. For sampling from the log-concave distribution of hypotheses, we utilize an implementation of the Metropolis-adjusted Langevin algorithm (MALA)  from TensorFlow\citep{tfa}. We compare OURS against three baselines

\textbf{1. Passive Learning (PASS):} Queries are selected according to a random permutation of the training set.

\textbf{2. Leverage Score Sampling (LSS):} Queries are sampled proportionally to the leverage scores of the training set, where for a dataset $X\in\mathbb{R}^{n\times d}$ the leverage score of the $i$th data point is defined as $\ell_i = \bigl[X(X^\top X)^{-1}X^\top\bigr]_{ii}$, which quantifies its importance in capturing the dataset's essential information~\citep{Chowdhury_Ramuhalli_2024, mahoney2011randomized}.

\textbf{3. Active Classification using Experimental Design (ACED)~\citep{katz2021improved}:} An active learning algorithm that has demonstrated superior empirical performance among active learning algorithms with provable sample complexity. %

Our evaluation follows a two-stage pipeline. First, we gather datasets by running the different algorithms, and then we train a logistic regression model on the queried datasets, assessing the model's performance on both the entire training set and a held-out test set. We perform experiments on two datasets: a synthetic dataset (syn\_100) and the Musk dataset~\citep{musk_(version_2)_75} (musk\_v2) described below.

\textbf{1. Synthetic Dataset (syn\_100):} The synthetic dataset, referred to as syn\_100, consists of points sampled uniformly from the hypercube $[-1, 1]^{100} \subseteq \mathbb{R}^{100}$. To enable the logistic regression model to include a bias term, we augment each data point with an additional dimension set to a constant value of $1$. We generate a random vector $w^* \in [-1, 1]^{101}$ and assign labels to each data point $x$ according to the probability $\Pr[\text{$x$ is labeled as $1$}] = \frac{1}{1 + \exp\left( -\bra{w^*}^\top x \right)}$. Both the training and test sets contain 10000 data points.

\textbf{2. Musk Dataset (musk\_v2)}~\citep{musk_(version_2)_75}: This dataset comprises 102 molecules, with 39 identified as musks and 63 as non-musks by human experts. The objective is to predict whether new molecules are musks or non-musks using 166 features that describe the molecules' various conformations. We split the dataset into training and test sets containing 4420 and 2178 data points, respectively.

We measure performance in terms of accuracy on both datasets. Additionally, on the synthetic dataset—where the ground truth is known—we evaluate performance using the weighted $\ell_2$-distances defined in Section~\ref{subsec:moti}. The experimental results are shown in Figure~\ref{fig:main}\footnote[1]{Error bars represent the standard errors.
} and Figure~\ref{fig:l2}\footnotemark[1]. Table~\ref{tab:num_queries} demonstrates that OURS achieves comparable training performance to other methods on some target value, while requiring significantly fewer queries on both datasets.

\begin{table}
    \centering
    \small
    \begin{tabularx}{\columnwidth}{@{}l c *{4}{>{\centering\arraybackslash}X}@{}}
        \toprule
        {\scriptsize \textbf{Dataset}} & {\scriptsize \textbf{Training}} & \multicolumn{4}{c}{{\scriptsize \textbf{Number of Queries}}} \\
        \cmidrule(lr){3-6}
        {} & {\scriptsize \textbf{Performance}} & {\scriptsize \textbf{OURS}} & {\scriptsize \textbf{ACED}} & {\scriptsize \textbf{PASS}} & 
        {\scriptsize \textbf{LSS}}\\
        \midrule
        Synthetic & 82.5\% & 601 & 889 & 961 & 1024\\
        Musk & 92.5\% & 249 & 762 & 676 & 656 \\
        \bottomrule
    \end{tabularx}
    \caption{Comparison of the number of queries needed for OURS, ACED, LSS, and PASS to achieve a specific performance.}
    \label{tab:num_queries}
\end{table}

\section{CONCLUSION}
We presented the first active logistic regression algorithm with provable label complexity bounds that is polynomially competitive with the optimal on any problem instance, up to  a factor $d^2$ and  polylogarithmic factors.
In particular, the promise of active learning is that it allows for algorithms like binary search that can, in some cases, improve the sample complexities from $\frac{1}{\eps}$ to $\log \frac{1}{\eps}$.  Whenever such an improvement is possible, our algorithm will get a bound of   $d^2\poly(\log \frac{dR_1R_2}{\eps})$.

\clearpage

\section*{ACKNOWLEDGMENTS}
We would like to thank Jiawei Li and Zhiyang Xun for their valuable suggestions in polishing the paper. We also thank ChatGPT for its assistance in refining our manuscript.  Support was provided by NSF award CCF-1751040 (CAREER) and the NSF AI Institute for Foundations of Machine Learning (IFML).

\bibliographystyle{abbrvnat}
\bibliography{ref}

\begin{thebibliography}{28}
\providecommand{\natexlab}[1]{#1}
\providecommand{\url}[1]{\texttt{#1}}
\expandafter\ifx\csname urlstyle\endcsname\relax
  \providecommand{\doi}[1]{doi: #1}\else
  \providecommand{\doi}{doi: \begingroup \urlstyle{rm}\Url}\fi

\bibitem[Abadi et~al.(2015)Abadi, Agarwal, Barham, Brevdo, Chen, Citro, Corrado, Davis, Dean, Devin, Ghemawat, Goodfellow, Harp, Irving, Isard, Jia, Jozefowicz, Kaiser, Kudlur, Levenberg, Man\'{e}, Monga, Moore, Murray, Olah, Schuster, Shlens, Steiner, Sutskever, Talwar, Tucker, Vanhoucke, Vasudevan, Vi\'{e}gas, Vinyals, Warden, Wattenberg, Wicke, Yu, and Zheng]{tfa}
M.~Abadi, A.~Agarwal, P.~Barham, E.~Brevdo, Z.~Chen, C.~Citro, G.~S. Corrado, A.~Davis, J.~Dean, M.~Devin, S.~Ghemawat, I.~Goodfellow, A.~Harp, G.~Irving, M.~Isard, Y.~Jia, R.~Jozefowicz, L.~Kaiser, M.~Kudlur, J.~Levenberg, D.~Man\'{e}, R.~Monga, S.~Moore, D.~Murray, C.~Olah, M.~Schuster, J.~Shlens, B.~Steiner, I.~Sutskever, K.~Talwar, P.~Tucker, V.~Vanhoucke, V.~Vasudevan, F.~Vi\'{e}gas, O.~Vinyals, P.~Warden, M.~Wattenberg, M.~Wicke, Y.~Yu, and X.~Zheng.
\newblock {TensorFlow}: Large-scale machine learning on heterogeneous systems, 2015.
\newblock URL \url{https://www.tensorflow.org/}.
\newblock Software available from tensorflow.org.

\bibitem[Balcan et~al.(2006)Balcan, Beygelzimer, and Langford]{balcan2006agnostic}
M.-F. Balcan, A.~Beygelzimer, and J.~Langford.
\newblock Agnostic active learning.
\newblock In \emph{Proceedings of the 23rd international conference on Machine learning}, pages 65--72, 2006.

\bibitem[Chapman and Jain(1994)]{musk_(version_2)_75}
D.~Chapman and A.~Jain.
\newblock {Musk (Version 2)}.
\newblock UCI Machine Learning Repository, 1994.
\newblock {DOI}: https://doi.org/10.24432/C51608.

\bibitem[Chen and Price(2019)]{chen2019active}
X.~Chen and E.~Price.
\newblock Active regression via linear-sample sparsification.
\newblock In \emph{Conference on Learning Theory}, pages 663--695. PMLR, 2019.

\bibitem[Chowdhury and Ramuhalli(2024)]{Chowdhury_Ramuhalli_2024}
A.~Chowdhury and P.~Ramuhalli.
\newblock A provably accurate randomized sampling algorithm for logistic regression.
\newblock \emph{Proceedings of the AAAI Conference on Artificial Intelligence}, 38\penalty0 (10):\penalty0 11597--11605, Mar. 2024.
\newblock \doi{10.1609/aaai.v38i10.29042}.
\newblock URL \url{https://ojs.aaai.org/index.php/AAAI/article/view/29042}.

\bibitem[Cohn et~al.(1994)Cohn, Atlas, and Ladner]{cohn1994improving}
D.~Cohn, L.~Atlas, and R.~Ladner.
\newblock Improving generalization with active learning.
\newblock \emph{Machine learning}, 15:\penalty0 201--221, 1994.

\bibitem[Dasgupta(2004)]{dasgupta2004analysis}
S.~Dasgupta.
\newblock Analysis of a greedy active learning strategy.
\newblock \emph{Advances in neural information processing systems}, 17, 2004.

\bibitem[Dasgupta et~al.(2007)Dasgupta, Hsu, and Monteleoni]{dasgupta2007general}
S.~Dasgupta, D.~J. Hsu, and C.~Monteleoni.
\newblock A general agnostic active learning algorithm.
\newblock \emph{Advances in neural information processing systems}, 20, 2007.

\bibitem[Gajjar et~al.(2023)Gajjar, Musco, and Hegde]{gajjar2023active}
A.~Gajjar, C.~Musco, and C.~Hegde.
\newblock Active learning for single neuron models with lipschitz non-linearities.
\newblock In \emph{International Conference on Artificial Intelligence and Statistics}, pages 4101--4113. PMLR, 2023.

\bibitem[Gajjar et~al.(2024)Gajjar, Tai, Xingyu, Hegde, Musco, and Li]{gajjar2024agnostic}
A.~Gajjar, W.~M. Tai, X.~Xingyu, C.~Hegde, C.~Musco, and Y.~Li.
\newblock Agnostic active learning of single index models with linear sample complexity.
\newblock In \emph{The Thirty Seventh Annual Conference on Learning Theory}, pages 1715--1754. PMLR, 2024.

\bibitem[Hanneke(2007)]{hanneke2007bound}
S.~Hanneke.
\newblock A bound on the label complexity of agnostic active learning.
\newblock In \emph{Proceedings of the 24th international conference on Machine learning}, pages 353--360, 2007.

\bibitem[Hsu and Mazumdar(2024)]{hsu2024sample}
D.~Hsu and A.~Mazumdar.
\newblock On the sample complexity of parameter estimation in logistic regression with normal design.
\newblock In \emph{The Thirty Seventh Annual Conference on Learning Theory}, pages 2418--2437. PMLR, 2024.

\bibitem[Kakade et~al.(2011)Kakade, Kanade, Shamir, and Kalai]{kakade2011efficient}
S.~M. Kakade, V.~Kanade, O.~Shamir, and A.~Kalai.
\newblock Efficient learning of generalized linear and single index models with isotonic regression.
\newblock \emph{Advances in Neural Information Processing Systems}, 24, 2011.

\bibitem[Kalai and Sastry(2009)]{kalai2009isotron}
A.~T. Kalai and R.~Sastry.
\newblock The isotron algorithm: High-dimensional isotonic regression.
\newblock In \emph{COLT}, volume~1, page~9, 2009.

\bibitem[Kannan et~al.(1997)Kannan, Lov{\'a}sz, and Simonovits]{kannan1997random}
R.~Kannan, L.~Lov{\'a}sz, and M.~Simonovits.
\newblock Random walks and an $o^*(n^5)$ volume algorithm for convex bodies.
\newblock \emph{Random Structures \& Algorithms}, 11\penalty0 (1):\penalty0 1--50, 1997.

\bibitem[Katz-Samuels et~al.(2021)Katz-Samuels, Zhang, Jain, and Jamieson]{katz2021improved}
J.~Katz-Samuels, J.~Zhang, L.~Jain, and K.~Jamieson.
\newblock Improved algorithms for agnostic pool-based active classification.
\newblock In \emph{International Conference on Machine Learning}, pages 5334--5344. PMLR, 2021.

\bibitem[Klivans and Meka(2017)]{klivans2017learning}
A.~Klivans and R.~Meka.
\newblock Learning graphical models using multiplicative weights.
\newblock In \emph{2017 IEEE 58th Annual Symposium on Foundations of Computer Science (FOCS)}, pages 343--354. IEEE, 2017.

\bibitem[Lattimore and Szepesv{\'a}ri(2020)]{lattimore2020bandit}
T.~Lattimore and C.~Szepesv{\'a}ri.
\newblock \emph{Bandit algorithms}.
\newblock Cambridge University Press, 2020.

\bibitem[Lov{\'a}sz and Simonovits(1993)]{lovasz1993random}
L.~Lov{\'a}sz and M.~Simonovits.
\newblock Random walks in a convex body and an improved volume algorithm.
\newblock \emph{Random structures \& algorithms}, 4\penalty0 (4):\penalty0 359--412, 1993.

\bibitem[Lov{\'a}sz and Vempala(2006)]{lovasz2006simulated}
L.~Lov{\'a}sz and S.~Vempala.
\newblock Simulated annealing in convex bodies and an $o^*(n^4)$ volume algorithm.
\newblock \emph{Journal of Computer and System Sciences}, 72\penalty0 (2):\penalty0 392--417, 2006.

\bibitem[Mahoney et~al.(2011)]{mahoney2011randomized}
M.~W. Mahoney et~al.
\newblock Randomized algorithms for matrices and data.
\newblock \emph{Foundations and Trends{\textregistered} in Machine Learning}, 3\penalty0 (2):\penalty0 123--224, 2011.

\bibitem[Mai et~al.(2021)Mai, Musco, and Rao]{mai2021coresets}
T.~Mai, C.~Musco, and A.~Rao.
\newblock Coresets for classification--simplified and strengthened.
\newblock \emph{Advances in Neural Information Processing Systems}, 34:\penalty0 11643--11654, 2021.

\bibitem[Munteanu et~al.(2018)Munteanu, Schwiegelshohn, Sohler, and Woodruff]{munteanu2018coresets}
A.~Munteanu, C.~Schwiegelshohn, C.~Sohler, and D.~Woodruff.
\newblock On coresets for logistic regression.
\newblock \emph{Advances in Neural Information Processing Systems}, 31, 2018.

\bibitem[Musco et~al.(2022)Musco, Musco, Woodruff, and Yasuda]{musco2022active}
C.~Musco, C.~Musco, D.~P. Woodruff, and T.~Yasuda.
\newblock Active linear regression for $\ell_p$ norms and beyond.
\newblock In \emph{2022 IEEE 63rd Annual Symposium on Foundations of Computer Science (FOCS)}, pages 744--753. IEEE, 2022.

\bibitem[Price and Zhou(2023)]{price2023competitive}
E.~Price and Y.~Zhou.
\newblock A competitive algorithm for agnostic active learning.
\newblock \emph{Advances in Neural Information Processing Systems}, 36, 2023.

\bibitem[Sabato and Munos(2014)]{sabato2014active}
S.~Sabato and R.~Munos.
\newblock Active regression by stratification.
\newblock \emph{Advances in Neural Information Processing Systems}, 27, 2014.

\bibitem[Vempala and Wibisono(2019)]{vempala2019rapid}
S.~Vempala and A.~Wibisono.
\newblock Rapid convergence of the unadjusted langevin algorithm: Isoperimetry suffices.
\newblock \emph{Advances in neural information processing systems}, 32, 2019.

\bibitem[Yang and Loog(2018)]{yang2018benchmark}
Y.~Yang and M.~Loog.
\newblock A benchmark and comparison of active learning for logistic regression.
\newblock \emph{Pattern Recognition}, 83:\penalty0 401--415, 2018.

\end{thebibliography}

\clearpage

 \appendix

 \renewcommand{\thesection}{\Alph{section}}

\onecolumn

\section{PROOFS FOR THE ALGORITHM ANALYSIS}\label{appendix:analysisproof}

 In the potential analysis, densities and integrals over the parameter space are taken with respect to the uniform distribution $\pi$ on $\Theta$.

\subsection{Proof of Lemma~\ref{lemma:mstarrelationstrong}}
\begin{proof}
At a high level, for any algorithm, we aim to construct two hypotheses $h_1$ and $h_2$ that are sufficiently far apart that the algorithm must distinguish them, yet are hard to distinguish. This gives a lower bound on   $m^*(X, \calD_X, H, \varepsilon, \delta)$ for $\delta\le0.1$. We first introduce some notation. Let $A$ be any algorithm such that with $m$ label queries, it returns a hypothesis that is $\varepsilon$-close to the ground truth with probability at least $0.9$ for any choice of the ground truth. Specifically, if two hypotheses $h_1$ and $h_2$ satisfy $\norm{h_1-h_2}\ge2\varepsilon$, then we can use $A$ to distinguish $h_1$ and $h_2$ with probability at least $0.9$. Let the random variable $P^{A}_h$ be the transcript of algorithm $A$ if $h$ is the ground truth. Note that $h$ could be improper, which means $h$ may not belong to the hypothesis class $H$. The transcript $P^{A}_h$ is a collection of pairs of queried points and labels of the form $\curly{\bra{x_1,y_1},\bra{x_2,y_2},\cdots}$. Without loss of generality, assume all transcripts have length $m$: if $A$ terminates before making $m$ queries, we can pad the transcript with arbitrary queries that $A$ ignores. Let $\bar{h}$ be the mean of the hypotheses, defined as $\bar{h}(x)\coloneqq\E_{h'\sim\lambda}\squr{h'(x)}$ for every $x$. Let $d$ denote the marginal of the distribution of $P^{A}_{\bar{h}}$, in other words, each $x$ is expected to be queried $m\cdot d(x)$ times. We define the query-induced KL distance between two hypotheses $h_1$ and $h_2$ with respect to query distribution $d$ as
\[
\dist_d(h_1,h_2)=\E_{x\sim d}\squr{\DKL(h_1(x)\Vert h_2(x))}. 
\]
\paragraph{Constructing two hard-to-distinguish hypotheses.} We aim to construct two hypotheses $h_1$ and $h_2$ with the following properties:
\begin{enumerate}
    \item The distance between the two hypotheses satisfies   $\norm{h_1-h_2}\ge2\varepsilon$, so $\mathcal{A}$ can distinguish them with probability at least 0.9.
    \item The query-induced distance between $h_1,h_2$ and $\bar{h}$ satisfies $\dist_d\bra{\bar{h},h_1}, \dist_d\bra{\bar{h},h_2}\le10\E_{h'\sim\lambda}\squr{\dist_d\bra{\bar{h},h'}}$, so they are both close to $\bar{h}$ and hard to distinguish.
\end{enumerate}
The expected query-induced KL distance between $\bar{h}$ and $h'$ sampled from $\lambda$ is $\E_{h'\sim\lambda}\squr{\dist_d\bra{\bar{h},h'}}=\E_{x\sim d}\E_{h'\sim\lambda}\squr{\DKL\bra{\bar{h}(x),h'(x)}}$. By Markov's inequality,
\begin{equation}\label{eq:distanceMarkov}
\Pr_{h'\sim \lambda}\squr{\dist_d(\bar{h},h')> 10\E_{h'\sim\lambda}\squr{\dist_d\bra{\bar{h},h'}}}\le\frac{1}{10}. 
\end{equation}
Choose $h_1$ to be any hypothesis satisfying   $\dist_d(\bar{h},h')\le 10\E_{h'\sim\lambda}\squr{\dist_d\bra{\bar{h},h'}}$. Furthermore, by the anti-concentration assumption of this lemma, we know that at least $20\%$ of the hypotheses are at least $2\varepsilon$ from $h_1$. Combining this with~(\ref{eq:distanceMarkov}), there exists a hypothesis $h_2$ such that $\norm{h_1-h_2}\ge 2\varepsilon$ and $\dist_d\bra{\bar{h},h_2}\le 10\E_{h'\sim\lambda}\squr{\dist_d\bra{\bar{h},h'}}$ as desired.

\paragraph{Implications.} We know that $A$ can distinguish $h_1$ and $h_2$ with more than 0.9 probability using $m$ queries and by the definition of total variation distance,
\begin{align*}
    0.1\ge\Pr\squr{\text{$A$ fails to distinguish $h_1$ and $h_2$}}\ge\frac{1}{2}\bra{1-\DTV\bra{P^{A}_{h_1},P^{A}_{h_2}}}.
\end{align*}
By the triangle inequality and Pinsker's inequality, we have
  \begin{align*}
    0.8\le \DTV\bra{P^{A}_{h_1},P^{A}_{h_2}}&\le \DTV\bra{P^{A}_{\bar{h}}, P^{A}_{h_1}}+\DTV\bra{P^{A}_{\bar{h}}, P^{A}_{h_2}}\\
    &\le\sqrt{\frac{1}{2}\DKL\bra{P^{A}_{\bar{h}}\Vert P^{A}_{h_1}}}+\sqrt{\frac{1}{2}\DKL\bra{P^{A}_{\bar{h}}\Vert P^{A}_{h_2}}}.
\end{align*} 
Decomposing the KL divergence using Lemma~\ref{lemma:divergencedecomp} \citep{lattimore2020bandit}[Lemma 15.1] gives
\begin{align*}
    \sqrt{\frac{1}{2}\DKL\bra{P^{A}_{\bar{h}}\Vert P^{A}_{h_1}}}+\sqrt{\frac{1}{2}\DKL\bra{P^{A}_{\bar{h}}\Vert P^{A}_{h_2}}}&=\sqrt{\frac{m}{2}\dist_d\bra{\bar{h},h_1}}+\sqrt{\frac{m}{2}\dist_d\bra{\bar{h},h_2}}\\    &\le5\sqrt{m\E_{h'\sim\lambda}\squr{\dist_d\bra{\bar{h},h'}}}\\
    &\le 5\sqrt{mr(x^*)},
\end{align*}
where $x^*=\argmax_{x} r_\lambda(x)$. Here the first inequality follows from the   choice of $h_1$ and $h_2$. The last step follows from the definition of $x^*$. As a result, we have
\begin{align*}
r(x^*)m\gtrsim 1.
\end{align*}
Since this result holds for any algorithm, including the optimal one, rearranging gives
  \[
m^*(X,\calD_X,H,\varepsilon,\delta)\gtrsim\frac{1}{r(x^*)}.
\] 
\end{proof}

\subsection{Proof of Lemma~\ref{lemma:singlequerypotentialgain}}
\begin{proof}
A direct calculation gives
  \begin{align*}
& \mathbb{E}_{y_i}\left[ \psi_{i+1}(h^*) - \psi_i(h^*) \,\middle|\, x_i \right] \nonumber \\
=& \mathbb{E}_{y_i}\left[ \log \frac{\lambda_{i+1}(h^*)}{\lambda_i(h^*)} \,\middle|\, x_i \right]\\
=& \mathbb{E}_{y_i}\left[ \log\left( \frac{w_{i+1}(h^*)}{w_i(h^*)} \cdot \frac{w_i(H)}{w_{i+1}(H)} \right) \,\middle|\, x_i \right]\\
=& \mathbb{E}_{y_i}\left[ \log\left( \exp\bra{-\ell_{h^*}(x_i, y_i)} \cdot \frac{1}{\mathbb{E}_{h \sim \lambda_i} \left[ \exp\bra{-\ell_h(x_i, y_i)} \right]} \right) \,\middle|\, x_i \right]\\
=& \mathbb{E}_{y_i}\left[ -\ell_{h^*}(x_i, y_i) - \log \mathbb{E}_{h \sim \lambda_i} \left[ \exp\bra{-\ell_h(x_i, y_i)} \right] \,\middle|\, x_i \right]\\
=& h^*(x_i) \log \frac{h^*(x_i)}{\bar{h}_{\lambda_i}(x_i)} + \left( 1 - h^*(x_i) \right) \log \frac{1 - h^*(x_i)}{1 - \bar{h}_{\lambda_i}(x_i)} \nonumber \\
=& D_{\mathrm{KL}}\left( h^*(x_i) \,\big\Vert\, \bar{h}_{\lambda_i}(x_i) \right).
\end{align*} 
\end{proof}

\subsection{Proof of Lemma~\ref{lemma:mainpotentialgrowthlb}}
\begin{proof}
Lemma~\ref{lemma:singlequerypotentialgain} gives the expected potential growth under a single query. To extend this to a double query, we need to calculate how the mean $\bar{h}_i(x_i)$ changes after the first query. To simplify notation, we omit the queried point $x_i$, which is clear from the context. Let $w'_i$ denote the weight function after the first query and $\lambda'_i$ denote the induced distribution. We also use $\bar{h}'_i$ to denote the mean under $\lambda'_i$.
\paragraph{Case 1: The first label $y_i^1=1$.} Then each hypothesis $h$ incurs the penalty $\ell_h(x_i,1)=-\log h(x_i)$ so $w'_i(h)=hw_i(h)$. Therefore, 
\begin{align*}
\bar{h}'_i &= \E_{h \sim \lambda'_i}\squr{h} = \int_{h \in H} \lambda'_i(h) h \, dh = \int_{h \in H} \frac{w'_i(h)}{w'_i(H)} h \, dh = \int_{h \in H} \frac{h \, w_i(h)}{\displaystyle \int_{h \in H} h \, w_i(h) \, dh} \, h \, dh \\
&= \int_{h \in H} \frac{h \, \lambda_i(h)}{\displaystyle \int_{h \in H} h \, \lambda_i(h) \, dh} \, h \, dh = \int_{h \in H} \frac{h^2 \, \lambda_i(h)}{\bar{h}_i} \, dh = \frac{1}{\bar{h}_i} \E_{h \sim \lambda_i}\squr{h^2} \\
&= \frac{1}{\bar{h}_i} \bra{ \var_{h \sim \lambda_i}\squr{h} + \bar{h}_i^2 } \\
&= \frac{ \var_{h \sim \lambda_i}\squr{h} }{ \bar{h}_i } + \bar{h}_i.
\end{align*}
\paragraph{Case 2: The first label $y_i^1=0$.} Then each hypothesis $h$ incurs the penalty   $\ell_h(x_i,0)=-\log \bra{1-h(x_i)}$  so $w'_i(h)=(1-h)w_i(h)$. Therefore, 
\begin{align*}
\bar{h}'_i &= \E_{h \sim \lambda'_i} \left[ h \right] = \int_{h \in H} \lambda'_i(h) h \, dh \\
&= \int_{h \in H} \frac{ w'_i(h) }{ w'_i(H) } h \, dh = \int_{h \in H} \frac{ (1 - h) w_i(h) }{ \displaystyle \int_{h \in H} (1 - h) w_i(h) \, dh } h \, dh = \int_{h \in H} \frac{ (1 - h) \lambda_i(h) }{ \displaystyle \int_{h \in H} (1 - h) \lambda_i(h) \, dh } h \, dh \\
&= \int_{h \in H} \frac{ (1 - h) h \lambda_i(h) }{ 1 - \bar{h}_i } \, dh = \frac{1}{1 - \bar{h}_i} \left( \bar{h}_i - \E_{h \sim \lambda_i} \left[ h^2 \right] \right) = \frac{1}{1 - \bar{h}_i} \left( \bar{h}_i - \var_{h \sim \lambda_i} \left[ h \right] - \bar{h}_i^2 \right) \\
&= - \frac{ \var_{h \sim \lambda_i} \left[ h \right] }{ 1 - \bar{h}_i } + \bar{h}_i.
\end{align*}
Case 1 occurs with probability $h^*$ and Case 2 with probability $1-h^*$. Combining these cases with Lemma~\ref{lemma:singlequerypotentialgain} gives
\begin{align*}
&\E_{y_i^1,y_i^2}\squr{\psi_{i+1}(h^*)-\psi_i(h^*)|x_i}\\
=&\DKL\bra{h^*(x_i),\bar{h}_i(x_i)}+h^*(x_i)\DKL\bra{h^*(x_i),\bar{h}_i(x_i)+\frac{\var_{h\sim\lambda_i}\squr{h}}{\bar{h}_i(x_i)}}\\
&+\bra{1-h^*(x_i)}\DKL\bra{h^*(x_i),\bar{h}_i(x_i)-\frac{\var_{h\sim\lambda_i}\squr{h}}{1-\bar{h}_i(x_i)}}
\end{align*}
By Pinsker's inequality, we have
\begin{align*}
&\DKL\left(h^*, \bar{h}_i\right) + h^* \DKL\left(h^*, \bar{h}_i + \frac{ \var_{h \sim \lambda_i} \left[ h \right] }{ \bar{h}_i } \right) + \left(1 - h^*\right) \DKL\left(h^*, \bar{h}_i - \frac{ \var_{h \sim \lambda_i} \left[ h \right] }{ 1 - \bar{h}_i } \right) \\
\gtrsim&\DTV^2\left(h^*, \bar{h}_i\right) + h^* \DTV^2\left(h^*, \bar{h}_i + \frac{ \var_{h \sim \lambda_i} \left[ h \right] }{ \bar{h}_i } \right) + \left(1 - h^*\right) \DTV^2\left(h^*, \bar{h}_i - \frac{ \var_{h \sim \lambda_i} \left[ h \right] }{ 1 - \bar{h}_i } \right) \\
\gtrsim&\left(h^* - \bar{h}_i\right)^2 + h^* \left( \frac{ \var_{h \sim \lambda_i} \left[ h \right] }{ \bar{h}_i } - \left(h^* - \bar{h}_i\right) \right)^2 + \left(1 - h^*\right) \left( \frac{ \var_{h \sim \lambda_i} \left[ h \right] }{ 1 - \bar{h}_i } - \left(\bar{h}_i - h^*\right) \right)^2 \\
\gtrsim&\left(h^* - \bar{h}_i\right)^2 + \left( \var_{h \sim \lambda_i} \left[ h \right] \right)^2.
\end{align*}
The last step follows from the facts that $\max\curly{h^*,1-h^*}\ge\frac{1}{2}$ and $a^2+b^2\gtrsim(a+b)^2$.
\end{proof}

\subsection{Proof of Lemma~\ref{lemma:goodcaseexpectation}}
\begin{proof}
  Condition on the complete inner-loop transcript of phase $i$. Its query sequence is then fixed, and the fresh labels used to update $\lambda_i$ are independent of that transcript. The order of this fixed batch  does not affect   its posterior. If $A_{(i,j)}$ holds, place the two queries at  $x_{(i,j)}$   first. Lemmas ~\ref{lemma:potentialfraction} and ~\ref{lemma:mainpotentialgrowthlb},   with
\[
\alpha=\frac{1}{C(m^*)^4\log^5(1/\gamma)},\qquad
\beta=\frac{1}{C(m^*)^4\log^4(1/\gamma)},
\] 
  give an expected potential gain $\Omega((m^*)^{-24}\log^{-32}(1/\gamma))$ from this pair. Every subsequent query has nonnegative conditional expected gain by  Lemma~\ref{lemma:singlequerypotentialgain}. Averaging over the remaining inner-loop randomness proves the claim.
\end{proof}

\subsection{Proof of Lemma~\ref{lemma:altpotentialgrowthlb}}
\begin{proof}
To simplify notation, we omit the phase index $i$. We first bound the expected potential change of   $\log \tilde p_j^{H\setminus B_{h'}(2\varepsilon)}(h^*)$  for any fixed $h'$ after a single query:
  \begin{align}
&\E_{y_j}\left[\log\frac{\tilde p_{j+1}^{H\setminus B_{h'}(2\varepsilon)}(h^*)}{\tilde p_j^{H\setminus B_{h'}(2\varepsilon)}(h^*)}\middle| x_j\right] \nonumber\\
=& \E_{y_j}\left[\log\bra{\frac{w_{j+1}(h^*)}{w_j(h^*)}\frac{w_j(H\setminus B_{h'}(2\varepsilon))}{w_{j+1}(H\setminus B_{h'}(2\varepsilon))}}\middle|x_j \right] \nonumber\\
=&\E_{y_j}\left[\log\bra{\frac{w_{j+1}(h^*)}{w_j(h^*)}\frac{1}{\displaystyle\int_{h\in H\setminus B_{h'}(2\varepsilon)}\frac{w_{j+1}(h)}{w_j({H\setminus B_{h'}(2\varepsilon)})}dh}}\middle|x_j \right] \nonumber\\
=&\E_{y_j}\left[\log\bra{\exp\bra{-\ell_{h^*}(x_j,y_j)}\frac{1}{\E_{h\sim p_j^{H\setminus B_{h'}(2\varepsilon)}}\squr{\exp\bra{-\ell_h(x_j,y_j)}}}}\middle|x_j \right] \nonumber\\
=&h^*(x_j)\log\frac{h^*(x_j)}{\E_{h\sim p_j^{H\setminus B_{h'}(2\varepsilon)}}\squr{h(x_j)}}+\bra{1-h^*(x_j)}\log\frac{1-h^*(x_j)}{1-\E_{h\sim p_j^{H\setminus B_{h'}(2\varepsilon)}}\squr{h(x_j)}} \nonumber\\
=&\DKL\bra{h^*(x_j),\bar{h}_{p_j^{H\setminus B_{h'}(2\varepsilon)}}(x_j)} \label{eq:potentialgrowthequality}.
\end{align} 
 Here the extended density $\tilde p$ retains its numerator even if $h^*$ lies outside the conditioning set.  Following the same steps as in the proof of Lemma~\ref{lemma:mainpotentialgrowthlb} and taking expectation over $\lambda^0$ completes the proof.
\end{proof}

\subsection{Proof of Lemma~\ref{lemma:badcaseexpectation}}
\begin{proof}
  Condition on the history before the current labels and on $\widebar A_{(i,j)}$. Write $L=\log(1/\gamma)$, omit the queried point and phase indices, and let $\hat\lambda=\frac{1}{2}\lambda^0+\frac{1}{2}\lambda^1$.

Lemma~\ref{Lemma:LambdaBarNoConcentration} shows that balls of radius $\varepsilon$ have small probability mass. Applying Lemma~\ref{lemma:mstarrelationstrong} at accuracy $\varepsilon/2$ and failure probability $0.02$, followed by  Lemma~\ref{lemma:Relationfmstar},   yields
\[
r_{\hat\lambda}(x)\gtrsim\frac{1}{m^*},
\qquad
\var_{\hat\lambda}[h]\gtrsim\frac{1}{(m^*)^2L^2}.
\] 
  Here Lemma~\ref{lemma:mstarrelationstrong} applies unchanged at failure probability $0.02$, since its proof only requires success probability $0.9$.

  By  Pinsker's inequality   and $\abs{h-h^*}\le1$, the complement of the witness event gives
\[
\E_{\lambda^0}[(h-h^*)^2]
\le \frac{1}{2C(m^*)^4L^5}
   +\frac{1}{C(m^*)^4L^4}.
\] 
 Since
  \[
\var_{\hat\lambda}[h]\le \E_{\hat\lambda}[(h-h^*)^2]
=\frac{1}{2}\E_{\lambda^0}[(h-h^*)^2]
 +\frac{1}{2}\E_{\lambda^1}[(h-h^*)^2],
\] 
  choosing the constant in the definition of $A_{(i,j)}$ sufficiently large implies
\[
\E_{\lambda^1}[(h-h^*)^2]\gtrsim\frac{1}{(m^*)^2L^2}.
\] 

  For each restricted distribution $p^{H\setminus B_{h'}(2\varepsilon)}$, let
$s=(\bar h-h^*)^2$ and $v=\var[h]$. As $s\le1$,
$s+v^2\ge(s+v)^2/2$.  Lemma~\ref{lemma:altpotentialgrowthlb}   and  Jensen's inequality   therefore give
\[
\E[\widetilde\psi_{(i,j+1)}-\widetilde\psi_{(i,j)}
 \mid\calF_{(i,j)},\widebar A_{(i,j)}]
\gtrsim
\left(\E_{\lambda^1}[(h-h^*)^2]\right)^2
\gtrsim\frac{1}{(m^*)^4L^4}.
\]
 \end{proof}

\subsection{Proof of Lemma~\ref{lemma:phaseguaranteegoodcase}}
\begin{proof}
  All probabilities and expectations below are conditional on the history at the start of the phase. Write $L=\log(1/\gamma)$,
$a=1/(C(m^*)^{24}L^{32})$ and $b=1/(C(m^*)^4L^4)$ for sufficiently large constants. Let
\[
Z_j=\mathds{1}_{A_{(i,j)}}+
       \widetilde\psi_{(i,j+1)}-\widetilde\psi_{(i,j)}.
\] 
  The alternative potential always has a nonnegative conditional expected increment. Lemma~\ref{lemma:badcaseexpectation} thus implies
$\E[Z_j\mid\calF_{(i,j)}]\ge b$.
A pair of clipped likelihood updates changes the alternative potential by at most $2L$, so the centered increments of $\sum_j Z_j$ are bounded by $O(L)$. The Azuma--Hoeffding inequality therefore gives
\[
\Pr\left[\sum_{j=1}^M Z_j<\beta\right]\le0.005
\] 
 when $M=C(\beta+1)(m^*)^8L^{10}$ for a sufficiently large absolute constant $C$.

 Let   $E=\bigcup_{j=1}^M A_{(i,j)}$. Since the initial density relative to the uniform prior is $1$, $\widetilde\psi_{(i,1)}(h^*)\ge0$. On $\overline E$, the indicators in $Z_j$ vanish. Hence, if $\Pr(E)\le0.005$,
  \[
\Pr[\widetilde\psi_{(i,M+1)}(h^*)\ge\beta]\ge0.99.
\] 
  Otherwise, condition on the full inner-loop transcript. Whenever $E$ occurs, reorder one witness pair to the start of the fresh outer batch, as in the proof of Lemma~\ref{lemma:goodcaseexpectation}. Its conditional expected gain is at least $a$, and the remaining queries have nonnegative expected gain. Thus
\[
\E[\psi_{i+1}(h^*)-\psi_i(h^*)]\ge a\Pr(E)>0.005a.
\] 
  This proves the two claimed alternatives.
 \end{proof}

\subsection{Proof of Lemma~\ref{lemma:tildepsi}}
\begin{proof}
Consider any ball \(B'\) with radius \(2\varepsilon\) whose center is at least \(3\varepsilon\) away from \(h^*\). Then \(B'\) does not intersect \(B\), implying that   \(\tilde{p}_{(i,j)}^{H \setminus B'}(B) \leq 1\) so $\log\tilde{p}_{(i,j)}^{H \setminus B'}(B)\le0$. Equivalently, if   \(\log\tilde{p}_{(i,j)}^{H \setminus C'}(B) > 0\) and $\tilde{p}_{(i,j)}^{H \setminus C'}(B) > 1$  for some ball \(C'\) of radius \(2\varepsilon\), then the center of \(C'\) must be at most \(3\varepsilon\) from \(h^*\). Suppose, for contradiction, that every ball of radius \(4\varepsilon\) containing \(h^*\) has probability mass less than 0.9. Then, for any ball $C'$ of radius $2\varepsilon$ whose center is at most $3\varepsilon$ from $h^*$,   \(\tilde{p}_{(i,j)}^{H \setminus C'}(B) = \frac{p_{(i,j)}(B)}{p_{(i,j)}(H \setminus C')} \leq \frac{p_{(i,j)}(C'')}{p_{(i,j)}(H \setminus C'')}\le\frac{0.9}{p_{(i,j)}(H \setminus C'')}\). Here, \(C''\) and \(C'\) share the same center, and \(C''\) has radius \(4\varepsilon\). Thus, $C''$ contains $B$ by definition. Moreover,   \(p_{(i,j)}(H \setminus C'') \geq 1 - 0.9=0.1\)  by our assumption, so   $\log\tilde{p}_{(i,j)}^{H \setminus C'}(B)\le\log 9$. Under this assumption,
  \begin{align*}
    \tilde{\psi}_{(i,j)}(B)=\E_{h'\sim\lambda^0}\squr{\log\tilde{p}_{(i,j)}^{H\setminus B_{h'}(2\varepsilon)}(B)}\le \Pr_{h'\sim\lambda^0}\squr{\norm{h'-h^*}_2\le3\varepsilon}\cdot\log 9<10.
\end{align*} 
This is a contradiction. Therefore, there exists a ball $C$ with radius $4\epsilon$ containing $h^*$ and   $p_{(i,j)}(C)\ge0.9$.
\end{proof}

\subsection{Proof of Lemma~\ref{lemma:clippedfinal}}
\begin{proof}
Let   $L=\log(1/\gamma)$, and
\[
A=\log\frac{e dR_1R_2m^*}{\varepsilon\gamma}.
\]
Set
\[
K=\left\lceil Cd(m^*)^{24}L^{32}A\right\rceil,\qquad
M=\left\lceil Cd(m^*)^8L^{10}A\right\rceil
\]
for sufficiently large absolute constants. Counting the inner queries and the fresh outer queries gives
\begin{equation}\label{eq:totalqueries}
T=4KM=O(d^2(m^*)^{32}L^{42}A^2).
\end{equation}
Use this deterministic query budget  in Lemma~\ref{lemma:ballprobability}, and let $U$ and the parameter neighborhood $B$ be defined there. Its formula gives $U=O(dA)$. The constants can therefore be chosen so that  Lemma~\ref{lemma:phaseguaranteegoodcase}   applies with $\beta=U+10$.

Call a phase bad if, conditional on its starting history, the first alternative of Lemma~\ref{lemma:phaseguaranteegoodcase} does not hold. Its conditional expected ordinary potential gain is then at least $a=1/(C(m^*)^{24}L^{32})$. Every phase has nonnegative expected gain. The uniform-prior normalization gives $\psi_1(h^*)=0$, while  Lemma~\ref{lemma:ballprobability}   and  $\lambda_{K+1}(B)\le1$   give $\psi_{K+1}(h^*)\le U$. Summing the expected increments gives
\[
a\,\E[\text{number of bad phases}]
\le\E[\psi_{K+1}(h^*)-\psi_1(h^*)]\le U.
\] 
  Taking the constant in $K$ sufficiently large makes the expected fraction of bad phases at most $0.01$.

In each remaining phase, conditional on its starting history,
$\widetilde\psi_{(i,M+1)}(h^*)\ge U+10$ with probability at least $0.99$.
 Lemma~\ref{lemma:ballprobability}   then gives
$\widetilde\psi_{(i,M+1)}(B)\ge10$.
By Lemma~\ref{lemma:tildepsi},   some radius-$4\varepsilon$ ball containing $h^*$ has $p_{(i,M+1)}$-mass at least $0.9$. That ball is contained in $B_{h^*}(8\varepsilon)$. Consequently,
\[
\E\left[\bar\lambda(B_{h^*}(8\varepsilon))\right]
\ge0.99\cdot0.99\cdot0.9>0.8.
\] 
  A hypothesis sampled from $\bar\lambda$ therefore has error at most $8\varepsilon$ with probability greater than $0.8$. Together with \eqref{eq:totalqueries}, this proves the claimed query bound.
 \end{proof}

\subsection{Proof of Lemma~\ref{lemma:blackboxred}}
\begin{proof}
  Couple each label from  $h^*$   with a label from
$g^*=\mathrm{clip}(h^*,\gamma)$ using the same independent uniform random variable. Conditional on agreement of the preceding transcripts, the next queried point is the same in both runs and the labels disagree with probability at most $\gamma$. Thus the transcripts agree throughout the run except with probability at most $\gamma T\le0.1$.

On the clipped run, the algorithm returns
$\hat g=\mathrm{clip}(h_{\hat\theta},\gamma)$,
which satisfies
$\|\hat g-g^*\|\le\eta$ with  probability at least   $0.8$.
On the original run, return the underlying hypothesis $h_{\hat\theta}\in H$.
The triangle inequality gives
\[
\|h_{\hat\theta}-h^*\|
\le\|h_{\hat\theta}-\hat g\|+\|\hat g-g^*\|+\|g^*-h^*\|
\le\eta+2\gamma.
\] 
 The coupling and a union bound give success probability at least $0.7$, using the same query budget $T$.
 \end{proof}

\subsection{Proof of Theorem~\ref{thm:finaltheorem}}
\begin{proof}
 Let   $m=\max\{1,m^*(X,\calD_X,H,\varepsilon/2,0.01)\}$ and
$L=\log(e dR_1R_2m/\varepsilon)$.
We first compare the clipped and original benchmarks. We run an optimal original-instance learner using labels from
$g^*=\mathrm{clip}(h^*,\gamma)$, and clip its output. The coupling in Lemma~\ref{lemma:blackboxred} gives transcript disagreement probability at most $\gamma m$. Since clipping is a pointwise contraction, $\gamma m\le0.01$ implies
\[
\max\{1,m^*(X,\calD_X,H_\gamma,\varepsilon/2,0.02)\}\le m.
\] 

  Choose
\[
\gamma=\frac{\varepsilon}{C d^2m^{32}L^{44}},\qquad
K=\left\lceil Cdm^{24}L^{33}\right\rceil,\qquad
M=\left\lceil Cdm^8L^{11}\right\rceil.
\] 
  Choose sufficiently large constants in $K,M$, and then a sufficiently large denominator constant in $\gamma$, so that
$\log(1/\gamma)=O(L)$  and
  $\log(e dR_1R_2m/(\varepsilon\gamma))=O(L)$.
Thus these choices satisfy Lemma~\ref{lemma:clippedfinal}, whose proof also permits an upper bound on the clipped benchmark in place of its exact value. They give a deterministic query budget
\[
T=4KM=O(d^2m^{32}L^{44}),
\] 
  with $\gamma T\le0.1$, $\gamma m\le0.01$, and $\gamma\le\varepsilon/2$.

  If the diameter shortcut applies, any returned clipped hypothesis has error at most $4\varepsilon$, and its underlying hypothesis has error at most $4\varepsilon+2\gamma$.
Otherwise, Lemma~\ref{lemma:clippedfinal} gives error $8\varepsilon$ with probability greater than $0.8$ on the clipped instance. Lemma~\ref{lemma:blackboxred} transfers this to the original oracle with error at most $8\varepsilon+2\gamma\le17\varepsilon$ and success probability at least $0.7$.
\end{proof}

\subsection{Proof of Corollary~\ref{corollary:highprobability}}
\begin{proof}
We run \( O\left( \log \frac{1}{\delta} \right) \) independent copies of Algorithm~\ref{Alg:RealMain}. Let \( A \) be the event that more than 60\% of the returned hypotheses \( \hat{h}_i \) satisfy \( \err\bra{\hat{h}_i} \leq 17\varepsilon \). By Theorem~\ref{thm:finaltheorem} and the Chernoff bound, \( A \) occurs with probability at least \( 1 - \delta \).

Conditioned on event \( A \), any ball of radius \( 34\varepsilon \) centered at a hypothesis more than \( 68\varepsilon \) away from \( h^* \) has probability at most 0.4, as it contains no hypothesis within \( 17\varepsilon \) of \( h^* \). Conversely, a ball of radius \( 34\varepsilon \) centered at a hypothesis with error at most \( 17\varepsilon \) has probability greater than 0.6. Therefore, conditioned on \( A \), selecting the hypothesis at the center of the heaviest ball of radius \( 34\varepsilon \) ensures that it is at most \( 68\varepsilon \) away from \( h^* \).
\end{proof}

\begingroup
\section{PROOF OF EXAMPLE~\ref{exp:exponential}}\label{appendix:exponential}
The sample complexity of Example~\ref{exp:exponential} is formally given by the following proposition.

\begin{proposition}\label{prop:exponential}
For the instance in Example~\ref{exp:exponential}, let $\varepsilon=1/(20\sqrt N)$ and $\delta=0.1$. To achieve error at most $\varepsilon$ with probability at least $1-\delta$, there is an active learner using $O(\log N\log\log(eN))$ labels, whereas any passive learner requires $\Omega(N)=\Omega(1/\varepsilon^2)$ labels in the worst case. Moreover, the algorithm in Corollary~\ref{corollary:highprobability} achieves the same error and success probability using $\polylog(1/\varepsilon)$ labels.
\end{proposition}

\begin{proof}
Let $h^*=h_{a,t}$ be the true hypothesis.

\paragraph{Active learning bound.}
We first use binary search to locate informative points near $t$, then fit the hypothesis using labels from these points. We maintain an interval initially equal to $[0,N]$. In each round with interval length greater than $2$, we repeatedly query its midpoint $k$, rounded down to an integer, to estimate the label probability $h^*(k)$ to accuracy $1/16$. If the empirical mean lies in $[7/16,9/16]$, we stop: an accurate estimate then implies $h^*(k)\in[3/8,5/8]$ and hence $\abs{k-t}\le\log(5/3)<1$. If the empirical mean is greater than $9/16$, we keep the interval to the left of $k$, since an accurate estimate implies $h^*(k)>1/2$ and $t<k$. If it is less than $7/16$, we similarly keep the interval to the right of $k$. We then repeat at the new midpoint. If the interval reaches length at most $2$, we stop and choose a nearest integer to its midpoint in $\{1,\ldots,N-1\}$.

There are $O(\log N)$ constant-accuracy estimates. By Hoeffding's inequality, $O(\log\log(eN))$ labels per estimate suffice for failure probability $O(1/\log N)$. A union bound gives $\abs{k-t}\le1$ with probability at least $0.95$, using $O(\log N\log\log(eN))$ labels in total.

We next estimate the label probabilities at $k$ and $k+1$, both within distance $2$ of $t$. Using $O(\eta^{-2})$ fresh labels, their empirical means $\hat p_k,\hat p_{k+1}$ are both accurate to within $\eta$ with probability at least $0.95$, where $\eta>0$ is a sufficiently small universal constant. We choose
\[
(\hat a,\hat t)\in\argmin_{\substack{a'\in[1,2]\\t'\in[k-1,k+1]\cap[0,N]}}
\max_{i\in\{k,k+1\}}\abs{h_{a',t'}(i)-\hat p_i},
\qquad \hat h=h_{\hat a,\hat t}.
\]
When all estimates are accurate, the true parameters are feasible and fit both empirical means within $\eta$. Thus the minimizer does too, giving $\abs{\hat h(i)-h^*(i)}\le2\eta$ for $i\in\{k,k+1\}$.

Let $u=a(k-t)$ and $\hat u=\hat a(k-\hat t)$. The sigmoid arguments at the two queried points are $u,a+u$ and $\hat u,\hat a+\hat u$, all in $[-4,4]$. Since the inverse sigmoid is Lipschitz on $[\sigma(-4),\sigma(4)]$, we obtain
\[
\abs{\hat a-a}+\abs{\hat u-u}\le C\eta.
\]
To extend this bound to other points, the mean value theorem gives the first-order estimate
\[
\hat h(i)-h^*(i)
=\sigma'(\xi_i)\bra{(\hat a-a)(i-k)+(\hat u-u)},
\]
where $\xi_i$ lies between $a(i-k)+u$ and $\hat a(i-k)+\hat u$. Both parameter pairs lie in $[1,2]\times[-2,2]$, so $\abs{\xi_i}\ge\abs{i-k}-2$. Together with $\sigma'(z)\le e^{-\abs{z}}$, this yields
\[
\abs{\hat h(i)-h^*(i)}
\le C\eta(1+\abs{i-k})e^{-\abs{i-k}}.
\]

Summing the squared errors under the uniform distribution gives
\[
\norm{\hat h-h^*}^2\le\frac{C\eta^2}{N}
\sum_{j\in\mathbb Z}(1+\abs{j})^2e^{-2\abs{j}}
\le\frac{C\eta^2}{N}.
\]
Choosing $\eta$ small enough gives $\err(\hat h)\le\varepsilon$. Both stages succeed with probability at least $0.9$ by a union bound, using $O(\log N\log\log(eN))$ labels in total. The same learner achieves error $\varepsilon/136$ and failure probability at most $0.01$ by choosing a smaller constant $\eta$ and using more labels for each estimate; the label complexity remains $O(\log N\log\log(eN))$, with a larger universal constant.

\paragraph{Passive learning.}
We construct two hypotheses that are more than $2\varepsilon$ apart but hard to distinguish using passive samples. Let $h_0=h_{1,s}$ and $h_1=h_{1,s+1}$ for an integer $1\le s<N$. At $i=s$, their predictions differ by $\sigma(0)-\sigma(-1)>1/10$, so $\norm{h_0-h_1}>1/(10\sqrt N)=2\varepsilon$.

Consider a passive learner using $n$ labels. We couple its samples under $h_0$ and $h_1$ by drawing the same uniform point $i\in\{1,\ldots,N\}$ and the same $U\sim\mathrm{Unif}[0,1]$, and setting $Y_j=\mathds{1}\{U\le h_j(i)\}$ for $j\in\{0,1\}$. The probability that the two labels differ is
\[
\frac1N\sum_{i=1}^N\bra{\sigma(i-s)-\sigma(i-s-1)}
=\frac{\sigma(N-s)-\sigma(-s)}{N}\le\frac1N.
\]
By a union bound, the two sequences of $n$ independent samples are identical with probability at least $1-n/N$. We also use the same internal randomness for the learner in both runs, so it returns the same hypothesis whenever the samples are identical. On the other hand, if the learner succeeds with probability at least $0.9$ under each hypothesis, both runs succeed with probability at least $0.8$. Since $\norm{h_0-h_1}>2\varepsilon$, their outputs must be different on this event. Hence $n/N\ge0.8$, which gives the lower bound $\Omega(N)=\Omega(1/\varepsilon^2)$.

\paragraph{Label complexity of our algorithm.}
Finally, we apply Corollary~\ref{corollary:highprobability} to this example. We write the inputs as $x_i=(i,1)$ and the parameters as $\theta=(a,b)$, where $b=-at$. Then $d=2$ and $\Theta=\{(a,b):1\le a\le2,\,-Na\le b\le0\}$ is compact and convex with nonempty interior. We can take $R_1=2\sqrt{N^2+1}$ and $R_2=\sqrt{N^2+1}$. By the active upper bound, the optimal label complexity at accuracy $\varepsilon/136$ and failure probability $0.01$ is $O(\log N\log\log(eN))$. Applying the corollary with accuracy parameter $\varepsilon/68$ and $\delta=0.1$, we get error at most $\varepsilon$ using $\polylog N=\polylog(1/\varepsilon)$ labels, as desired.
\end{proof}
\endgroup

 \section{COMPLEMENTARY LEMMAS AND DEFINITIONS}
The following divergence decomposition lemma is adapted from \citet[Lemma 15.1]{lattimore2020bandit}. Although originally stated in the context of bandit problems, it applies directly to our setting, as our problem can be viewed as a special case of the bandit problem, where each \( x \) corresponds to an arm with a Bernoulli distribution.
\begin{lemma}[Divergence Decomposition]\label{lemma:divergencedecomp}
Let $\nu=\bra{P_1,\cdots,P_k}$ be the reward distributions associated with one $k$-armed bandit, and let $\nu'=\bra{P'_1,\cdots,P'_k}$ be the
reward distributions associated with another $k$-armed bandit. Fix a policy $\pi$ and let $\mathbb{P}_{\nu}=\mathbb{P}_{\nu\pi}$ and $\mathbb{P}_{\nu'}=\mathbb{P}_{\nu'\pi}$ be the probability measures on the canonical
bandit model (Section 4.6 in \citet{lattimore2020bandit}) induced by the n-round interconnection of $\pi$ and $\nu$ (respectively, $\pi$ and $\nu'$). Then,
\[
\DKL\bra{\mathbb{P}_{\nu}\Vert \mathbb{P}_{\nu'}}=\sum_{i=1}^k\E_\nu\squr{T_i(n)}\DKL\bra{P_i\Vert P_i'}.
\]
\end{lemma}

The following lemma gives an upper bound on the KL divergence under the clipping assumption.
\begin{lemma}\label{lemma:KLupperbound}
For $p,q\in\squr{\gamma,1-\gamma}$,
\[
\DKL\bra{p\Vert q}\lesssim \abs{p-q}\log\frac{1}{\gamma}.
\]
\end{lemma}
\begin{proof}
Without loss of generality, assume that $q\le p$. We consider two cases.\\
\paragraph{Case 1: $p-q\le \frac{1}{2}p.$} In this case, we have
\begin{align*}
    \DKL(p\Vert q)\le p\log\frac{p}{q}=-p\log\bra{\frac{p-\bra{p-q}}{p}}=-p\log\bra{1-\frac{p-q}{p}}\le2\abs{p-q}.
\end{align*}
The last inequality follows from $1-x\ge\exp\bra{-2x}$ for $x\le\frac{1}{2}$.\\
\paragraph{Case 2: $p-q>\frac{1}{2}p.$} In this case, we have
\[
\DKL\bra{p\Vert q}\le p\log\frac{p}{q}\le p\log\frac{1}{\gamma}\le2\abs{p-q}\log\frac{1}{\gamma}. 
\]
\end{proof}

The following lemma relates the proportion of ``bad hypotheses'' (far from $h^*$) to the potential growth.
\begin{lemma}\label{lemma:potentialfraction}
Under Assumption~\ref{assumption:clipping}, if
  \[
\Pr_{h\sim\lambda}\left[
 \DKL(h^*(x)\Vert h(x))\ge\alpha
\right]\ge\beta,
\] 
then
  \[
(\bar h_\lambda(x)-h^*(x))^2+
 \left(\var_{h\sim\lambda}[h(x)]\right)^2
\gtrsim\frac{\beta^2\alpha^4}{\log^4(1/\gamma)}.
\] 
\end{lemma}
\begin{proof}
  To simplify notation, we omit  $x$   and let $L=\log(1/\gamma)$.  Lemma~\ref{lemma:KLupperbound}   implies that on the indicated event,
$\abs{h-h^*}\ge\alpha/(CL)$ for an absolute constant $C$. Hence
\[
q:=\E_\lambda[(h-h^*)^2]\ge
 \frac{\beta\alpha^2}{C^2L^2}.
\] 
  Writing $s=(\bar h_\lambda-h^*)^2\le1$ and
$v=\var_\lambda[h]$, we have   $q=s+v$ and
\[
s+v^2\ge s^2+v^2\ge\frac{1}{2}(s+v)^2
\ge\frac{\beta^2\alpha^4}{2C^4L^4}.
\] 
\end{proof}

The following lemma shows that $\hat{\lambda}_{(i,j)}$ is not too concentrated for any $(i,j)$.
\begin{lemma}\label{Lemma:LambdaBarNoConcentration}
Let $\hat{\lambda}_{(i,j)} = \frac{1}{2}\lambda^0 + \frac{1}{2}\lambda^1_{(i,j)}$ be the distribution defined in Algorithm~\ref{Alg:Main}. Then, for any $h \in H$, the probability that $\hat{\lambda}_{(i,j)}$ assigns to the ball $B_h(\varepsilon)$ is at most $0.8$; that is,
\[
\hat{\lambda}_{(i,j)}\left( B_h(\varepsilon) \right) \le 0.8.
\]
\end{lemma}
\begin{proof}
 On the nontrivial branch of Algorithm~\ref{Alg:Main}, every hypothesis has another hypothesis at distance greater than $2\varepsilon$, so the conditional distributions below are well-defined.
We interpret the sampling procedure as follows:

\begin{itemize}
    \item With some probability distribution, we obtain a pair of hypotheses $(h_1, h_2)$.
    \item We then output one of these hypotheses, chosen uniformly at random.
\end{itemize}
For every such pair $(h_1, h_2)$, the hypotheses are   more than  $2\varepsilon$ apart; that is,   $\| h_1 - h_2 \| > 2\varepsilon$. This implies that neither $h_1$ nor $h_2$ lies within a radius-$\varepsilon$ ball centered at the other hypothesis. Consider any fixed radius-$\varepsilon$ ball $B \subseteq H$. Given a pair $(h_1, h_2)$, the probability that a randomly selected hypothesis from the pair lies within $B$ is at most $0.5$. This is because at most one of $h_1$ or $h_2$ can be in $B$, since they are   more than  $2\varepsilon$ apart. Taking the expectation over all possible pairs $(h_1, h_2)$ and applying the law of total probability, we conclude that for any $h \in H$:
\[
\hat{\lambda}_{(i,j)}\left( B_h(\varepsilon) \right) \le 0.5<0.8.
\]
This completes the proof.
\end{proof}

The following lemma relates the information function $r_{\lambda}$ to the lower bound on the expected potential gain in Lemma~\ref{lemma:mainpotentialgrowthlb}.
\begin{lemma}\label{lemma:Relationfmstar}
Under Assumption~\ref{assumption:clipping}, for any distribution $\lambda$ over $H_\gamma$ and any $x\in X$:
\begin{align*}
\bra{\var_{h\sim\lambda}\squr{h(x)}}^2\gtrsim\frac{r_\lambda^4(x)}{\log^4\frac{1}{\gamma}}.
\end{align*}
\end{lemma}
\begin{proof}
To simplify notation, we omit $x$ and use $h$ to denote $h(x)$. By Lemma~\ref{lemma:KLupperbound} and Jensen's inequality,
\[
r^2=\bra{\E_{h\sim\lambda}\squr{\DKL\bra{\bar{h},h}}}^2\le\E_{h\sim\lambda}\squr{\DKL^2\bra{\bar{h},h}}\lesssim\E_{h\sim\lambda}\squr{\bra{\bar{h}-h}^2\log^2\frac{1}{\gamma}}=\var_{h\sim\lambda}\squr{h}\log^2\frac{1}{\gamma}.
\]
Consequently,
\[
\bra{\var_{h\sim\lambda}\squr{h}}^2\gtrsim\frac{r^4}{\log^4\frac{1}{\gamma}}.
\]
\end{proof}

The   next lemma converts a posterior density at the true parameter into mass near that parameter.
\begin{lemma}\label{lemma:ballprobability}
Let   $T\ge1$ be a deterministic upper bound on the number of label updates and set
\[
r=\min\left\{\frac{\varepsilon}{2R_2},
                  \frac{\gamma}{TR_2}\right\},\qquad
B=\Theta\cap B_{\theta^*}(r),\qquad
U=1+d\log\frac{2R_1}{r}.
\]
For every posterior $\lambda$ formed from at most $T$ updates starting from $\pi$:
\begin{enumerate}
\item   $B\subseteq B_{h^*}(\varepsilon)$.
 \item   $\log\lambda(B)\ge\log\lambda(h^*)-U$.
 \item For   each inner-loop posterior,
$\widetilde\psi_{(i,j)}(B)\ge
 \widetilde\psi_{(i,j)}(h^*)-U$.
 \end{enumerate}
 Here $\lambda(h^*)$ denotes its density at $\theta^*$, and
$\widetilde\psi_{(i,j)}(B)=
\E_{h'\sim\lambda^0}\log
[p_{(i,j)}(B)/p_{(i,j)}(H\setminus B_{h'}(2\varepsilon))]$.
 \end{lemma}
 \begin{proof}
  Clipping preserves the sigmoid's Lipschitz bound, so for $\theta\in B$ and $x\in X$,
$\abs{h_\theta(x)-h_{\theta^*}(x)}\le R_2r\le\varepsilon/2$.
This  proves the first   assertion.

Each label likelihood lies in $[\gamma,1-\gamma]$. Since the logarithm is $1/\gamma$-Lipschitz there, the log-likelihood difference between $\theta$ and $\theta^*$ over at most  $T$   updates is bounded by
$TR_2r/\gamma\le1$. Thus $\lambda(\theta)\ge e^{-1}\lambda(\theta^*)$ throughout $B$.

For $t=r/(2R_1)\le1$, convexity and boundedness give
$\theta^*+t(\Theta-\theta^*)\subseteq B$.
This contracted copy has $\pi$-mass $t^d$, including when $\theta^*$ lies on the boundary of $\Theta$. Integrating the density bound therefore gives
\[
\lambda(B)\ge e^{-1}
 \left(\frac{r}{2R_1}\right)^d\lambda(\theta^*),
\] 
  which proves the second assertion. Dividing by
$p_{(i,j)}(H\setminus B_{h'}(2\varepsilon))$ preserves this inequality. Taking logarithms and averaging over $h'\sim\lambda^0$ proves the third assertion.
\end{proof}

 \end{document}